\documentclass{article} 
\usepackage{iclr2027_conference,times}

\usepackage{amsmath,amsfonts,bm}

\def\eqref#1{equation~\ref{#1}}

\def\1{\bm{1}}

\DeclareMathAlphabet{\mathsfit}{\encodingdefault}{\sfdefault}{m}{sl}
\SetMathAlphabet{\mathsfit}{bold}{\encodingdefault}{\sfdefault}{bx}{n}

\usepackage{appendix}
\usepackage{hyperref}
\usepackage[nameinlink]{cleveref}
\crefname{equation}{Eq.}{Eqs.}
\crefname{enumi}{Question}{Questions}
\crefname{appendix}{Appendix}{Appendices}
\crefname{assumption}{Assumption}{Assumptions}
\crefname{lemma}{Lemma}{Lemmas}
\crefname{section}{Section}{Sections}
\crefname{figure}{Figure}{Figures}

\Crefname{appendixsubsection}{Appendix}{Appendices}
\crefname{appendixsubsection}{appendix}{appendices}

\usepackage[table]{xcolor}
\usepackage{caption}
\definecolor{oursbg}{RGB}{245,247,250}
\definecolor{linkblue}{RGB}{35,82,140}

\definecolor{linkblue}{RGB}{0,80,180}

\hypersetup{
    colorlinks=true,
    linkcolor=linkblue,
    citecolor=linkblue,
    urlcolor=linkblue
}

\AddToHook{cmd/appendix/before}{%
    \crefalias{section}{appendix}%
}

\usepackage{url}
\usepackage[most]{tcolorbox}
\usepackage{amsthm}

\author{
Minseok Jeong$^{1}$ \quad
Hyewon Choi$^{1}$ \quad
Hiroyasu Tsukamoto$^{2}$ \quad
Soojean Han$^{1}$ \\
\vspace{-0.25em}
{\footnotesize
$^{1}$KAIST
\qquad
$^{2}$University of Illinois Urbana--Champaign
} \\
\vspace{-0.45em}
{\footnotesize
\texttt{\{dsa950115,chw110825,soojean\}@kaist.ac.kr}
\quad
\texttt{hiroyasu@illinois.edu}
}
}

\usepackage{url}
\usepackage[most]{tcolorbox}
\usepackage{amsthm}

\newtheorem{theorem}{Theorem}[section]

\newtheorem{lemma}[theorem]{Lemma}
\newtheorem{proposition}[theorem]{Proposition}

\theoremstyle{definition}
\newtheorem{definition}[theorem]{Definition}
\newtheorem{assumption}[theorem]{Assumption}

\theoremstyle{remark}

\title{The Linear Representation Hypothesis \\ for Vision-Language-Action Models}

\iclrfinalcopy
\begin{document}

\maketitle
\lhead{}

\begin{abstract}
The \emph{linear representation hypothesis} (LRH) has become a standard lens for measuring and intervening on semantic information through the internal representations of large language models (LLMs). A growing body of work has begun extending this perspective to vision-language-action (VLA) models, but the dynamical nature of embodied interaction introduces an additional challenge. Unlike semantic attributes commonly studied in LLMs, such as gender or language, a physical quantity of interest (QoI) in a VLA evolves jointly with the system dynamics: the representation influences the actions selected by the policy, which alter the physical state and, in turn, the next representation.

In this paper, we develop a theoretical, \emph{signature}-based formulation of the LRH for VLA that unifies representations and policies. On the representation side, we establish the existence of representations from which the future evolution of a QoI under a candidate action trajectory can be recovered via \emph{linear probing}. On the policy side, we introduce a \emph{signature generalized linear model} for stochastic action chunks. This structure yields a monotonic change in the expected future QoI along linear paths in natural parameter space, enabling \emph{linear steering}. We first validate this structure in a controlled oracle setting, then examine whether the same probing and steering mechanisms emerge in a pretrained VLA.
\end{abstract}

\section{Introduction}
The linear representation hypothesis (LRH) in large language models (LLMs) states that high-level concepts, such as gender and language, have representations that occupy particular linear directions, or low-dimensional subspaces, in their representation space~\citep{mikolov2013linguistic,elhage2022superposition,park2024linear}. Its appeal is that this geometry provides a simple mathematical interface to otherwise opaque models. After identifying a linear representation corresponding to a concept, it can be used in two ways: \emph{probing} measures the concept by projecting representations onto that direction~\citep{alain2017probing,belinkov2022probing}, while \emph{steering} intervenes on the concept by moving representations along it~\citep{li2023iti,turner2023actadd,zou2023repe}. These two operations have become a standard interface to understand the internal structures of LLMs, supported by substantial empirical evidence~\citep{li2023othello,marks2024geometry,templeton2024scaling} and, more recently, by systematic theoretical analysis~\citep{park2024linear, jiang2024origins, park2025geometry, garg2026features}.

It is natural to seek the same interface for robot control with vision- language-action (VLA) models, where the motivation is even stronger. VLA models~\citep{brohan2023rt2,kim2024openvla,octo2024,black2024pi0} are large, expensive to fine-tune, and trained on demonstration data that is costly to collect~\citep{oneill2024openx,khazatsky2024droid}; retraining is rarely an option, e.g., when one needs to understand what the agent has inferred from the scene or when a policy must be made more cautious. The representation space offers a cheaper point of access. If task-relevant physical information is encoded in a representation space and can be measured, then we can use it without additionally instrumenting the environment or modifying the model weights. If behavior is steerable, it can be adjusted at inference time without collecting new robot demonstrations. Recent work has increasingly built on this premise~\citep{lu2025probing,haon2025mechanistic,zhang2026frozen,miao2026coast}, bringing much of the LLM-side LRH framework into VLAs: physical concepts and their representations, together with probing and steering.

But the objects of interest are not quite the same. In an LLM, commonly probed semantic directions, such as gender  \(\texttt{Rep}(\text{``Female''})-\texttt{Rep}(\text{``Male''})\) or language  \(\texttt{Rep}(\text{``English''})-\texttt{Rep}(\text{``French''})\), are relatively static. In a VLA, by contrast, many quantities of interest pertain to physical systems and, crucially, their evolutions: the pose of a manipulated object, or the force at a contact.
A fixed representation direction may suffice to read out such a quantity at the current state, but decision-making often depends on how that quantity changes under action. For example, what matters is whether the object will reach a desired height or whether the contact force will not become excessive throughout task execution. Such future quantities depend not only on the current representation, but also on the action trajectory executed from that state and the dynamics through which those actions unfold. 
Although recent empirical studies on VLA models have shown promising results~\citep{lu2025probing,buurmeijer2026observingcontrollingfeaturesvisionlanguageaction}, a rigorous theoretical framework for characterizing this dynamical coupling remains largely undeveloped relative to the LRH literature for LLMs.

\begin{figure}
    \centering
    \includegraphics[width=0.94\linewidth]{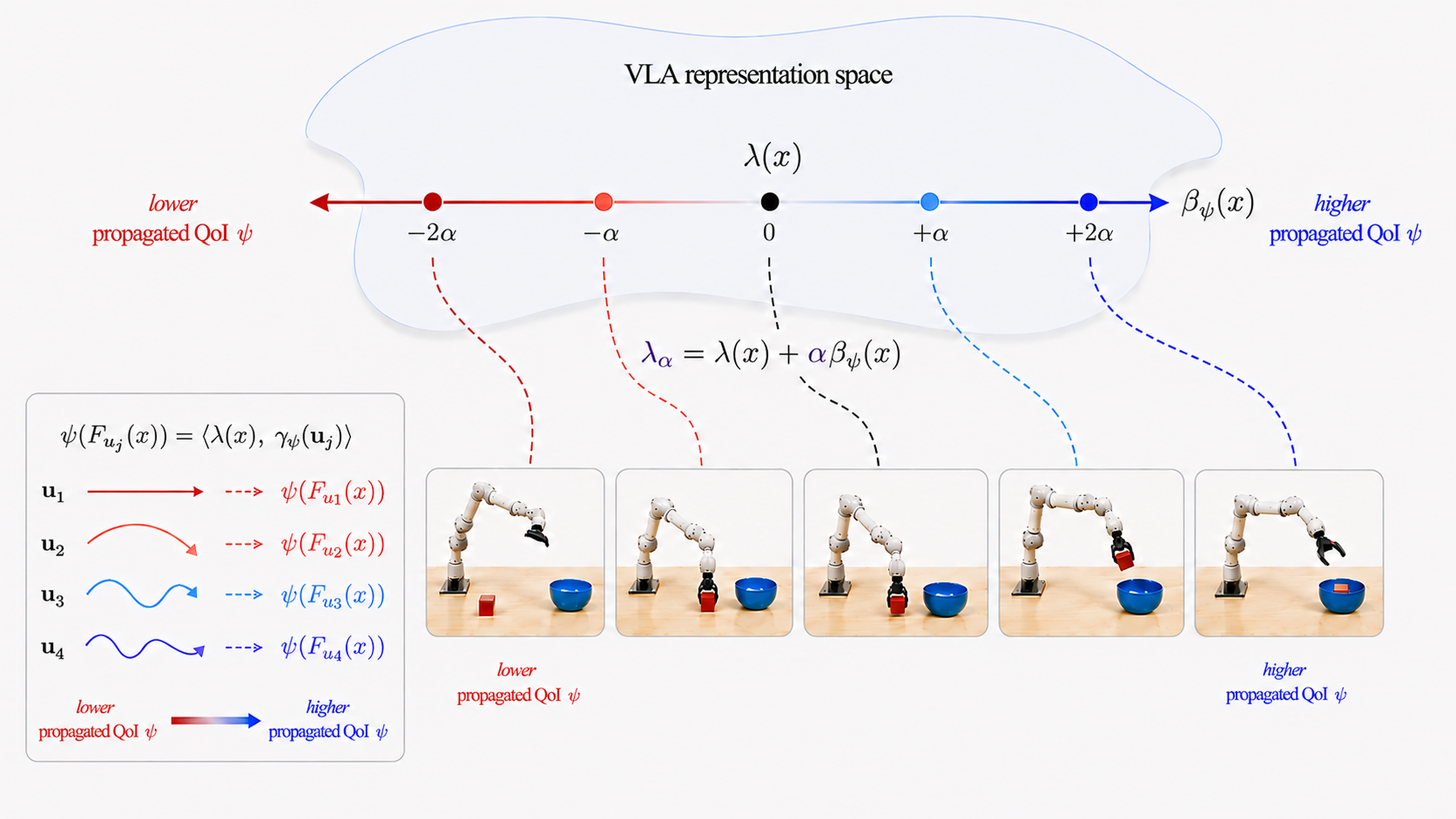}
    \caption{\textbf{Linear probing and steering in VLA representations.} Given a current state $x$ with representation $\lambda(x)$, candidate action trajectories $\mathbf{u}_j$ define action-conditioned linear probes that predict the propagated quantity of interest (QoI), $\psi(F_{\mathbf{u}_j}(x))=\langle \lambda(x), \gamma_\psi(\mathbf{u}_j)\rangle$. Steering the representation along $\lambda_\psi(x)$, via $\lambda(x)\mapsto\lambda(x)+\alpha\lambda_\psi(x)$, monotonically shifts the propagated QoI, allowing the policy to favor lower or higher physical outcomes.}
    \label{fig:placeholder}
\end{figure}

We therefore develop a mathematical abstraction of VLA models centered on the representation--action--dynamics interface. Our contributions are as follows.



\begin{enumerate}
    \item \textbf{A control-theoretic formulation of the LRH for VLA models.}
    We connect VLA models to the underlying controlled dynamical system, leading to VLA-specific notions of \emph{probing} and \emph{steering}. Building on these notions, we formulate the LRH for VLA models and identify the key challenges arising from their architectural differences from LLMs.
    
    \item \textbf{Shared linear representations for linear probing.} We show constructively that an ideal VLA representation can simultaneously encode the information of multiple physical quantities of interest, with each propagated value under a candidate action trajectory admitting a \emph{linear probe} in representation space.

    \item \textbf{Generalized linear policies for linear steering.}
    We introduce a path signature-based exponential family formulation of VLA policies, whose geometry provides a dually flat statistical manifold. Building on this geometry, we show that  \emph{linear probing} gives rise to \emph{linear steering} in the representation space.

    \item \textbf{Validation from oracle models to pretrained VLAs.} We first verify the predicted probing and steering mechanism exactly in a controlled oracle construction. We then test the framework on a pretrained OpenVLA-OFT, where both \emph{linear probing} and \emph{steering} remain competitive with substantially more flexible nonlinear alternatives.
\end{enumerate}
   


\paragraph{Basic background on VLA models.}
We consider a minimal abstraction of a VLA model. Let \(x\) denote the underlying physical state, which is not directly accessible to the model. Instead, the VLA receives an observation, \(o = h(x)\), which may include visual inputs, multiple camera views, or proprioceptive robot state. Given a language instruction \(\ell\), the VLA backbone maps \(o\) to an internal representation \(\bar{\lambda}(o,\ell)\). Throughout the paper, we fix \(\ell\) and suppress its dependence. Motivated by recent work on physically meaningful visual representations~\citep{bardes2024vjepa,assran2025vjepa2}, we conjecture that such information is 
accessible from the VLA representation. We therefore define \(\lambda(x) := \bar{\lambda}(h(x),\ell)\). The representation \(\lambda(x)\) is then passed to a stochastic action policy \(\pi\), such as a diffusion-based action-chunk policy. For our analysis, we model the resulting action chunk as a continuous-time open-loop trajectory over a horizon \(H\), \(\mathbf{u} := \{\mathbf u(t)\}_{t\in[0,H]} \sim \pi(\cdot\mid\lambda(x))\). The resulting VLA pipeline is summarized as
\begin{equation*}
\boxed{
x
\xrightarrow{h}
o
\xrightarrow{\bar{\lambda}(\cdot,\ell)}
\lambda(x)
\xrightarrow{\pi}
\mathbf{u}\sim\pi(\cdot\mid\lambda(x)).}
\end{equation*}
We ask whether the LRH extends to this largely unexplored architectural setting.

\section{A Control-Theoretic Formulation of VLA Models}
\subsection{VLA Models as Controlled Dynamical Systems}

We now connect the VLA model to the underlying physical dynamics by interpreting
the action trajectory \(\mathbf{u}\) as a control input to a nonlinear control-affine system:
\begin{subequations}\label{eq:controlled_system}
\begin{align}
    \dot{x}(t)
    &=
    f_0(x(t))
    +
    \sum_{i=1}^{m} f_i(x(t)) u^i(t)
    \label{eq:controlled_system_state}
    \\
    o(t)
    &=
    h(x(t))
    \label{eq:controlled_system_output}
\end{align}
\end{subequations}
Here, \(x(t) \in \mathcal{X}\subseteq \mathbb R^n\) denotes the unobserved physical state,
\(\mathbf u(t) = (u^1(t),\ldots,u^m(t)) \in \mathcal{U}\subset \mathbb{R}^m\) denotes the control input, \(f_0\) is the drift vector field, \(f_i\) is the control vector field associated with the \(i\)-th control input, and \(h:\mathcal{X} \to \mathcal{O}\subseteq \mathbb{R}^o\) is the observation map. Let \(\mathcal{U}_H\) denote the set of admissible open-loop control paths
\(\mathbf{u}:[0,H]\to\mathcal{U}\). For each
\(\mathbf{u}\in\mathcal{U}_H\), we denote by
\(F_{\mathbf{u}}:\mathcal{X}\to\mathcal{X}\) the corresponding
state-transition map over the horizon \(H\), defined by $x(H)=F_{\mathbf{u}}(x(0))$.

We next define a hidden scalar quantity $\psi:\mathcal{X}\to\mathbb{R}$ on the physical state space analogous to a high-level concept in the LLM setting.  


\begin{definition}[Quantity of Interest (QoI)]
Let \(\psi:\mathcal{X}\to\mathbb{R}\) be a real-valued function of the physical state \(x\). For example, \(\psi(x)=x_i\) may extract the \(i\)-th component of the state, such as a coordinate of a robot's position. 
More generally, \(\psi\) may represent any physically meaningful scalar quantity derived from \(x\), such as a control Lyapunov function or a value function.\footnote{
This general notion of a QoI may also support the study of safety and planning in representation space.
}
\end{definition}

\subsection{Probing and Steering in VLA Models}
We now adapt the notions of \emph{probing} and \emph{steering} to the VLA setting.
Unlike LLM representations, VLA representations are coupled to actions that induce physical state transitions through interaction with the environment. We must therefore reformulate probing and steering to reflect this interaction.

For \emph{probing}, we consider a physically meaningful QoI \(\psi(x)\) and ask
whether its future value under a candidate action trajectory $\mathbf{u}$ can be inferred from the VLA representation.

\begin{tcolorbox}[
    title=\textbf{Probing Problem},
    colback=white,
    colframe=black,
    boxrule=0.6pt,
    arc=1.5mm
]
Given a QoI \(\psi:\mathcal{X}\to\mathbb{R}\) and a candidate action trajectory \(\mathbf{u}\in\mathcal{U}_H\), can the propagated QoI
\[
\psi\!\left(F_{\mathbf{u}}(x)\right)
\]
be predicted using only the VLA representation \(\lambda(x)\) and \(\mathbf{u}\), without explicit access to the underlying physical state or dynamics?
\end{tcolorbox}

A representation is informative for probing if, together with \(\mathbf{u}\), it is sufficient to predict the action-conditioned propagated QoI \(\psi(F_{\mathbf{u}}(x))\). 

Existing VLA probing methods focus on static QoI, asking whether \(\psi(x)\) can be decoded directly from the internal representation~\citep{lu2025probing,bhardwaj2026decoding}, with promising empirical results.
However, static probing does not capture the action-conditioned evolution of the physical quantity. This distinction is particularly important for VLA models, whose latent representations both generate state-transitioning actions and provide a space for steering them~\citep{haon2025mechanistic}.\footnote{This is closely related to LeCun's view that an agentic system should be able to predict the consequences of contemplated actions \citep{lecun2022path}. A similar perspective is embodied in explicit world-model formulations, such as \citep{ha2018world}. Importantly, however, such predictive capability need not rely on an explicit world model and may also emerge implicitly within VLA models.}


For \emph{steering}, we ask whether an intervention in the representation space
can alter the action trajectory generated by the policy so as to induce a desired change in
the propagated physical state.

\begin{tcolorbox}[
    title=\textbf{Steering Problem},
    colback=white,
    colframe=black,
    boxrule=0.6pt,
    arc=1.5mm
]
Given a QoI \(\psi:\mathcal{X}\to\mathbb{R}\), can we modify the VLA representation $\lambda(x)$ so that the resulting policy generates an action trajectory $\widetilde{\mathbf{u}}
\sim
\pi\!\left(\cdot\mid\mathtt{Steer}(\lambda(x))\right)$ that drives the propagated QoI
\[
\psi\!\big(F_{\widetilde{\mathbf{u}}}(x)\big)
\]
toward a desired value?
\end{tcolorbox}

In general, a steering map \(\lambda(x) \mapsto \texttt{Steer}(\lambda(x))\) need not admit a simple algebraic form.
Recent work on VLA steering has therefore explored more structured interventions on internal representations, including conceptor-based subspace operators and learned distribution-matching maps that go beyond a single additive steering direction~\citep{miao2026coast,khayatan2026dimas}.

Together, these two problems characterize the role of the representation \(\lambda\) in a VLA model: \emph{probing} asks whether the representation contains sufficient information to predict the effect of a candidate action, whereas \emph{steering} asks whether modifying the representation can systematically
alter that effect through the policy.

\section{Linear Probing and Steering in VLA Models}

The linear representation hypothesis posits that both problems above admit a common resolution: linear structure in the VLA representation space. We extend the LRH for LLMs~\citep{park2024linear} to VLA models by characterizing two complementary notions: \emph{linear probing} and \emph{linear steering}.

\begin{definition}[Linear Probing in VLA]
\label{def:linear_probing}
Let \(\psi:\mathcal{X}\to\mathbb{R}\) be a QoI. We say that \(\psi\) admits
linear probing in the representation \(\lambda:\mathcal{X}\to\mathbb{R}^d\) if
there exists a map $\gamma_\psi:\mathcal{U}_H\to\mathbb{R}^d$  such that
\begin{equation*}
\psi\!\left(F_{\mathbf{u}}(x)\right)
=
\left< \lambda(x), \gamma_\psi(\mathbf{u})\right>_{\mathbb{R}^d},
\qquad
\forall\,x\in\mathcal{X},\ \mathbf{u}\in\mathcal{U}_H.
\end{equation*}
\end{definition}
Note that this structure differs from the LLM probing setting. There, the probe is typically given by a single global direction $\gamma_\psi\in \mathbb R^d$, and the QoI is estimated from that direction alone; in our setting, the corresponding object depends on the action trajectory $\mathbf{u}$. Such dependence again reflects the interaction structure discussed above. The usual static probing setting is recovered by restricting to a no-transition case, in which action dependence is no longer required. 


\paragraph{Challenges in linear probing.}
In the LLM setting, the relevant linear
measurement direction $\gamma_{\psi}$ is naturally tied to the rows of the output projection (or unembedding) matrix preceding the softmax. In contrast, VLA models do not provide an analogous, predefined measurement structure for propagated physical QoIs. A key challenge is therefore to characterize an action-conditioned map $\gamma:\mathcal{U}_H\to\mathbb{R}^d$ that associates action trajectory with the representation space.


\begin{definition}[Linear Steering in VLA]
Let \(\psi:\mathcal{X}\to\mathbb{R}\) be a QoI. We say that \(\psi\) admits
linear steering at \(x\in\mathcal{X}\) if there exists a nonzero vector \(\lambda_\psi(x)\in\mathbb{R}^d\) such that
\[
\alpha \longmapsto
\mathbb{E}_{\mathbf{u}\sim
\pi(\cdot\mid\lambda(x)+\alpha\lambda_\psi(x))}
\left[
    \psi\!\left(F_{\mathbf{u}}(x)\right)
\right]
\quad\text{is nondecreasing.}
\]
\label{def:linear_steering}
\end{definition}

In our formulation, unlike standard LLM steering approaches that rely on a fixed global direction $\lambda_\psi \in \mathbb{R}^d$, the steering direction $\lambda_\psi(x)$ is generally state-conditioned. This is a natural extension in the VLA setting, since the propagated QoI $\psi(F_{\mathbf u}(x))$ depends on both the current state and the action trajectory through the system dynamics; consequently, the same change in action can have different effects on the propagated QoI at different states.

\paragraph{Challenges in linear steering.}
To control a physical QoI through a linear translation \(\lambda \mapsto \lambda+\alpha\lambda_\psi\), we need to understand how such an intervention changes the downstream policy distribution and, consequently, the resulting physical behavior. To characterize this relationship, we adopt the information-geometric perspective of the \emph{Softmax} \citep{park2026information}, viewing $\lambda \mapsto \pi(\cdot\mid\lambda)$ as a parameterization of a family of VLA policies. This motivates us to seek a generalized linear model (GLM) parameterization of the VLA policy under which the geometry induced by variations in $\lambda$ admits an explicit characterization. However, owing to their structural differences from LLMs, modern generative VLA policies do not naturally admit such a convenient parameterization.

\section{Existence of Representation and Measurement Maps}
\label{sec:probing}
In this section, we take a closer look at \emph{linear probing} in VLA models and address two fundamental theoretical questions as follows:

\begin{enumerate}
    \item \label{q:shared-representation}
    Can the VLA representation \(\lambda\) capture the physical information necessary to characterize multiple quantities of interest, \(\Psi=\{\psi_j\}_{j=1}^p\)?

    \item \label{q:measurement-map}
    For each QoI \(\psi_j\), can we construct an action-conditioned measurement
    map \(\gamma_{\psi_j}\) such that the propagated QoI admits a linear
    probe from the VLA representation \(\lambda\)?
\end{enumerate}


Before stating our main result, we introduce a key object used in
its construction. Analogous to the role of unembedding vectors in the
softmax-based measurement of LLMs, the action-conditioned measurement map
\(\gamma\) can be chosen to factor through a \emph{signature} of the action trajectory.
\begin{definition}[Signature of an Action Trajectory]
Let \(\mathcal{A}:=\{0,1,\ldots,m\}\) denote the alphabet of action indices, with
\(u^0(t)\equiv 1\), and let \(\mathcal{A}^{\star}\) denote the set of all finite words (or multi-indices) over \(\mathcal{A}\), including the empty word \(\emptyset\). For a word \(I=(i_1,\ldots,i_k)=:i_1\cdots i_k\in\mathcal{A}^{\star}\), define
\begin{equation*}
\label{def:signature}
S(\mathbf{u})^I_{0,H}
:=
\int_{0\le t_1\le\cdots\le t_k\le H}
u^{i_1}(t_1)\cdots u^{i_k}(t_k)
\,\mathrm{d}t_1\cdots \mathrm{d}t_k,
\end{equation*}
with \(S(\mathbf{u})^{\emptyset}_{0,H}=1\). The collection $S(\mathbf{u}) := \bigl(S(\mathbf{u})^I_{0,H}\bigr)_{I\in\mathcal{A}^{\star}}\in \mathcal{S}$ is called the \emph{signature} of $\mathbf u$.
\end{definition}
The signature provides a principled representation of a continuous-time path,
encoding its geometric and sequential structure through iterated integrals.
It has been widely studied in mathematics, machine learning, and control
~\citep{lyons1998differential,kidger2019deep,kidger2020neural,bank2025stochastic}. We refer the reader to~\citep{chevyrev2026primer} for further details.

To address the challenges in linear probing described above, we interpret the signature \(S:\mathcal{U}_H \to \mathcal{S}\) as an infinite-dimensional realization of the action-conditioned measurement map $\gamma$. For each QoI \(\psi_j\), we consider a finite-dimensional measurement map that can be obtained through a linear map
\[
L_{\psi_j}:\mathcal{S}\to\mathbb{R}^d,
\qquad
\gamma_{\psi_j}(\mathbf{u})
:=
(L_{\psi_j}\circ S)(\mathbf{u}).
\]
The range of \(L_{\psi_j}\) therefore specifies a QoI-specific linear subspace of the representation space, while the signature parameterizes the measurement direction within that subspace as a function of the action trajectory. Our main result below establishes the existence of a shared VLA representation $\lambda$ where
these QoI-specific measurement structures can coexist.

\begin{theorem}[Existence of Representation and Measurement Maps]
\label{thm:existence}
Consider the control-affine system in \cref{eq:controlled_system_state}, 
a finite collection of QoIs $\Psi=\{\psi_j\}_{j=1}^p$, and a compact set \(K\subset\mathcal{X}\). Suppose the uniform Chen--Fliess conditions (\cref{app:chen-fliess}) hold on the horizon-\(H\) reachable set from \(K\). Then for every \(\varepsilon>0\), there exist \(d\in\mathbb{N}\), a shared representation map $\lambda:K\to\mathbb{R}^d$, and linear maps $L_{\psi_j}:\mathcal{S}\to\mathbb{R}^d,
\, j=1,\ldots,p$, such that, with $\gamma_{\psi_j} := L_{\psi_j}\circ S : \mathcal{U}_H\to\mathbb{R}^d$,
we have
\begin{equation*}
\max_{1\le j\le p}
\sup_{(x,\mathbf{u})\in K\times\mathcal{U}_H}
\left|
\psi_j(F_{\mathbf{u}}(x))
-
\left\langle
\lambda(x),
\gamma_{\psi_j}(\mathbf{u})
\right\rangle_{\mathbb{R}^d}
\right|
<
\varepsilon.
\end{equation*}
\end{theorem}
The result is established constructively via the classical Chen--Fliess expansion \citep{chen1977iterated,fliess1981fonctionnelles}, as detailed in \Cref{pf:existence}. 
Importantly, we do not assume that the learned VLA representation \(\lambda(x)=\bar{\lambda}(h(x),\ell)\) coincides exactly with the Lie-derivative-based construction. Rather, the latter establishes the existence of physically meaningful structure required by our framework. 
\Cref{thm:existence} provides a theoretical framework consistent with empirical linear probing results observed in existing VLA representations~\citep{lu2025probing,buurmeijer2026observingcontrollingfeaturesvisionlanguageaction}, while also providing an oracle against which arbitrary learned representations can be quantitatively assessed.


\Cref{thm:existence} answers
\Cref{q:shared-representation,q:measurement-map} affirmatively, in an existential sense. A desirable VLA model can admit a shared representation \(\lambda\) that supports multiple QoIs, with each QoI equipped with an action-conditioned linear measurement of the form $\gamma_{\psi_j}=L_{\psi_j}\circ S$. Thus, the action dependence is captured by the signature, with \(L_{\psi_j}\) specifying the QoI-specific measurement structure.

\section{Geometry of the Representation Space}
\label{sec:steering}
\Cref{sec:probing} addressed the representation-side question of whether propagated QoIs can admit action-conditioned linear probes.
Linear probing alone, however, does not determine how an intervention on the representation changes the downstream action distribution. 
Thus, we next introduce a geometry on the representation space for formulating \emph{linear steering}. To this end, we view stochastic action-chunk policies, such as autoregressive and diffusion policies, as defining a family of conditional distributions over action trajectories $\mathbf u \in \mathcal{U}_H$
\[
\mathcal{P}:=\left\{\pi(\cdot\mid \lambda):\mathcal{U}_H\to\mathbb{R}_{\geq 0}\right\},
\]
and regard this family as parameterized by \(\lambda\). To emphasize that the policy is parameterized by the representation \(\lambda\), we henceforth write the policy as \(\pi_\lambda\). Motivated by the geometric structure induced by the softmax-based distribution family in~\citep{park2026information}, we would like to introduce a corresponding notion of geometry in representation space through $\mathcal{P}$.

For the policy-side analysis, we adopt the following identification. In general, the distribution family induced by a modern generative policy may possess a highly complex parametric structure. To obtain a tractable yet expressive representation of path-dependent statistics, we lift action trajectories $\mathbf{u}\mapsto S(\mathbf{u})$ into the signature space $\mathcal{S}$. This construction leverages the universality of linear functionals on signatures for approximating continuous path functionals~\citep{lyons2014rough}.

We then place this construction in a Hilbert-space setting by replacing the finite-dimensional representation space $\mathbb{R}^d$ with the Hilbert signature space $\mathcal{H}_S$, obtained by completing $\mathcal{S}$ under a suitable inner product $\langle\cdot,\cdot\rangle_{\mathcal{H}_S}$. This formulation naturally bridges the representation and policy modules. Specifically, we regard $\lambda(x){\,\in\,}\Theta{\,\subseteq\,}\mathcal{H}_S$ simultaneously as the VLA representation and as the natural parameter of the signature exponential family. This identification allows us to draw on the well-developed body of work on linear representations in LLMs, including their geometry, and steering~\citep{park2024linear,marks2024geometry,park2026information}. 

\begin{definition}[Signature Exponential Family] Given a base measure \(\mu\) on the action-chunk space \(\mathcal{U}_H\), we define the log-partition function $A(\lambda):= \log \int_{\mathcal{U}_H} \exp(\left<\lambda, S(\mathbf{u}) \right>_{\mathcal H_S}) \mathrm{d}\mu(\mathbf{u})$, whenever the integral is finite. The \emph{signature exponential family} generated by $\mu$ is
\begin{equation*}
\mathcal{P}_{\mu} := \left\{\frac{\mathrm{d}\pi_\lambda}{\mathrm{d}\mu}(\mathbf{u}) = \exp(\left<\lambda, S(\mathbf{u})\right>_{\mathcal H_S} - A(\lambda)) : \lambda \in \Theta \right\}
\end{equation*}
where $\Theta := \{\lambda\in \mathcal H_S : A(\lambda) < \infty\}$ is the natural parameter space.
\end{definition}
The resulting exponential-family structure further admits a natural GLM perspective~\citep{nelder1972generalized,mccullagh1989generalized}, in which the conditional distribution of $\mathbf{u} \,|\, x$ belongs to an exponential family with natural parameter determined by $x$.
\begin{definition}[Signature Generalized Linear Model]
Suppose that \(\mathbf{u}\,|\, x\) follows a signature exponential family with base measure $\mu$ on $\mathcal{U}_H$. A
\emph{signature generalized linear model}  specifies a state-conditioned natural parameter map $\lambda:\mathcal{X}\to\Theta$, which induces the conditional distribution
\begin{equation*}
\frac{\mathrm{d}\pi_{\lambda}}{\mathrm{d}\mu}(\mathbf{u}\mid x)
=
\exp\!\left(
\left\langle \lambda(x),S(\mathbf{u})\right\rangle_{\mathcal{H}_S}
-
A(\lambda(x))
\right).
\end{equation*}
\label{def:signatureGLM}
\end{definition}


We introduce the signature GLM as an architecture-agnostic statistical abstraction linking representation changes to downstream action distributions, rather than as a characterization of specific VLA architectures. By abstracting away the internal generative mechanism, it provides a mathematically interpretable geometry for analyzing representation--policy interactions. Specifically, this exponential family structure endows the statistical manifold with a \emph{dually flat geometry}~\citep{amari2000methods}, with dual coordinate
$\varphi(\lambda):=\nabla A(\lambda)=\mathbb{E}_{\mathbf{u}\sim\pi_\lambda}[S(\mathbf{u})]$
and dual space $\varphi:=\varphi(\Theta)$. For a fixed state $x\in\mathcal{X}$, define the
state-conditioned steering direction $\lambda_\psi(x):=L_\psi^*\lambda(x)\in\mathcal{H}_S$ as in \Cref{def:linear_steering},
so that the expected propagated QoI is linear in the dual coordinate:
\begin{equation*}
\mathbb{E}_{\mathbf{u}\sim\pi_\lambda}
\bigl[\psi(F_{\mathbf{u}}(x))\bigr]
=
\left\langle
\lambda(x),
L_\psi\,\mathbb{E}_{\mathbf{u}\sim\pi_\lambda}
[S(\mathbf{u})]
\right\rangle_{\mathcal{H}_S}
=
\left\langle
L_\psi^*\lambda(x),
\varphi(\lambda)
\right\rangle_{\mathcal{H}_S}
=
\left\langle
\lambda_\psi(x),
\varphi(\lambda)
\right\rangle_{\mathcal{H}_S}.
\end{equation*}
A desired QoI level $c\in\mathbb{R}$ therefore corresponds to the dual affine
hyperplane $\varphi_{\psi,x}(c):=\{\varphi\in\varphi:\langle\lambda_\psi(x),\varphi\rangle_{\mathcal{H}_S}=c\}$,
and we formulate steering as the reverse-KL projection onto it,
\begin{equation*}
\widehat{\lambda}\in
\operatorname*{arg\,min}_{\varphi(\lambda)\in\varphi_{\psi,x}(c)}
D_{\mathrm{KL}}\bigl(\pi_\lambda\,\|\,\pi_{\lambda(x)}\bigr).
\end{equation*}
By dual Bregman geometry the solution is $\widehat{\lambda}=\lambda(x)+\alpha\lambda_\psi(x)$
for some $\alpha\in\mathbb{R}$: projecting onto a constant expected-QoI hyperplane
induces the linear path
\begin{equation}
\label{eq:linear-path}
\lambda_\alpha:=\lambda(x)+\alpha\,\lambda_\psi(x)
\end{equation}
in the natural parameter space. For readers seeking further details, we provide a complete derivation in \Cref{app:kl-projection}. We now show that this path satisfies the desired notion of linear steering.

\begin{theorem}[QoI Monotonicity under Linear Steering]
\label{thm:linear_steering}
Fix $x{\,\in\,}\mathcal{X}$, and let $\lambda_\alpha$ denote the steering path in \cref{eq:linear-path}. Define $J_{\psi,x}(\alpha){\,:=\,}\mathbb{E}_{\mathbf{u}\sim\pi_{\lambda_\alpha}} \bigl[\psi(F_{\mathbf{u}}(x))\bigr]$.
Suppose QoI $\psi$ admits an exact linear probe at $x$ of the form $\psi\bigl(F_{\mathbf{u}}(x)\bigr) = \bigl\langle
\lambda(x),\,L_\psi S(\mathbf{u})
\bigr\rangle_{\mathcal{H}_S}, \, \forall\mathbf{u}\in\mathcal{U}_H$, for a bounded linear operator $L_\psi{\,:\,}\mathcal{H}_S\to\mathcal{H}_S$.
Then, for every $\alpha$ such that $\lambda_\alpha\in\Theta$, $J_{\psi,x}(\alpha)$ is nondecreasing in $\alpha$.
\end{theorem}

See~\Cref{pf:linear_steering} for the proof. Although we assume an exact linear probe, \Cref{thm:existence} shows that this regime can be approximated arbitrarily well with a sufficiently expressive representation as \(d\) becomes large. We therefore interpret the exact-probe assumption as an idealized limit in which the propagated QoI can be represented without approximation error.
In this limit, the exact linear-steering guarantee is
recovered. Establishing quantitative steering guarantees under finite probing error is a separate robustness question that we leave for future work.

\section{Experimental Results}
\label{subsec:oracle-experiment}

\subsection{Numerical Verification}


We begin with a controlled numerical verification of the probing and steering theory developed
in~\cref{sec:probing,sec:steering}. To isolate the predicted steering mechanism from
representation-learning error, we instantiate the Chen--Fliess construction of~\Cref{thm:existence} as an \emph{oracle representation} for the controlled system, which yields an exact action-conditioned linear probe together with the corresponding signature exponential family policy. Due to space limitations, several intermediate steps are omitted here. Please see~\Cref{app:oracle-derivation} for the full derivation. 

\paragraph{Setup.}
Given the language instruction $\ell = $ ``\textit{Navigate to the goal.}'', consider the planar
control-affine system with state $x = (p_x, p_y, \theta) \in \mathbb{R}^3$ and input
$\mathbf{u} = (1, v, \omega)\in \mathcal{U}_H$,
\[
    \dot p_x = v,\qquad \dot p_y = \theta v,\qquad \dot\theta = \omega,
\]
over horizon $H=1$ from $x_0=(0,0,0)$, with nominal goal $p_{\mathrm{goal}}=(1,0)$. The candidate set
consists of $671$ piecewise-constant chunks that turn and then counter-turn, producing lateral motion while
approximately restoring the terminal heading.

\paragraph{Exact oracle probe and policy.}
We take $\psi(x)=p_y$. Only the words $\{\emptyset, 1, 21\}$ have nonzero Chen--Fliess coefficients, so
the expansion terminates exactly, hence the term \emph{oracle representation}:
\[
    \psi(F_\mathbf{u}(x))
    = \langle \lambda_\psi(x),S(\mathbf{u})\rangle_{\mathcal H_S} = p_y + \theta S(\mathbf{u})^1_{0,H} + S(\mathbf{u})^{21}_{0,H},
    \qquad
    \lambda_\psi(x)=\bigl[\,p_y,\;\theta,\;1\,\bigr]^\top .
\]
To obtain a representation that is not aligned with the QoI, we set $\lambda(x)=M\lambda_\psi(x)$ for an
invertible $M$ and $L_\psi=M^{-\top}$, $\gamma_\psi(\mathbf{u})=L_\psi S(\mathbf{u})$, so that
$\psi(F_\mathbf{u}(x)) = \langle \lambda(x),\gamma_\psi(\mathbf{u})\rangle_{\mathcal H_S}$ holds exactly
for every candidate chunk. The representation \(\lambda\) then plays two complementary roles: paired with
\(\gamma_\psi(\mathbf{u})\), it decodes the propagated QoI; paired with \(S(\mathbf{u})\), it is the natural
parameter of the policy
\[
    \pi_\lambda(\mathbf{u}\mid x_0)\propto\mu(\mathbf{u})\exp(\langle\lambda,S(\mathbf{u})\rangle_{\mathcal{H}_S}),
    \qquad
    \mu(\mathbf{u})\propto\exp(-C_{\mathrm{task}}(\mathbf{u};x_0)/\tau),
\]
where $C_{\mathrm{task}}$ is a goal-reaching cost and $\tau = 0.20$; the unsteered policy $\pi_{\lambda(x_0)}$ coincides with $\mu$.

\paragraph{Goal-directed VLA policy.} We define a discrete base measure $\mu$ over candidate actions with action preference encoded by the goal-reaching cost
    $C_{\mathrm{task}}(\mathbf{u};x_0) = 8\bigl(p_x(H)-1\bigr)^2 + 2p_y(H)^2 + 4\theta(H)^2 + 0.05\int_0^H\!\bigl(v^2+\omega^2\bigr)\mathrm{d}t.$
We set $\mu(\mathbf{u})\propto\exp(-C_{\mathrm{task}}(\mathbf{u};x_0)/\tau)$ with $\tau = 0.20$, and instantiate the signature exponential family policy as $\pi_\lambda(\mathbf{u}\mid x_0)\propto\mu(\mathbf{u})\exp(\langle\lambda,S(\mathbf{u})\rangle_{\mathcal{H}_S})$. 

\paragraph{Linear steering.}
Following~\Cref{thm:linear_steering}, we define $\lambda_\psi(x)=L_\psi^\top\lambda(x)$ and traverse
the linear path $\lambda_\alpha=\lambda(x_0)+\alpha\lambda_\psi(x_0)$. Since the log-density is affine in
$\lambda$ and $\langle\lambda_\psi(x_0),S(\mathbf{u})\rangle_{\mathcal H_S}=\psi(F_\mathbf{u}(x_0))$,
\[
    \pi_{\lambda_\alpha}(\mathbf{u}\mid x_0)
    \propto \pi_{\lambda(x_0)}(\mathbf{u}\mid x_0)\,
    \exp\bigl(\alpha\,\psi(F_{\mathbf{u}}(x_0))\bigr),
\]
so the steered policy is the nominal one tilted by the propagated QoI itself, and we report
$J_{\psi,x_0}(\alpha)$.

\paragraph{Results.}
Steering the representation $\lambda_\alpha$ moves the physical quantity it decodes: as $\alpha$ increases over $\{0,4,8\}$, the rollouts shift laterally and the expected propagated QoI rises monotonically from $0.000$ to $+0.134$ to $+0.255$, with the probe exact to $3\times10^{-16}$ on all $671$ candidates. The numerical slope matches $\mathrm{Var}_{\pi_{\lambda_\alpha}}[\psi(F_\mathbf{u}(x_0))]$ to $10^{-11}$ (\Cref{fig:oracle_slope}), so the response is predicted by a quantity of the nominal policy, at a KL cost of $0$, $0.27$ and $0.98$ nats. The steerable range is nonetheless bounded. Because $\psi$ is bounded on the finite candidate set, large $\alpha$ concentrates $\pi_{\lambda_\alpha}$ on the extremal chunks and collapses this variance; $J_{\psi,x_0}$ is monotone but saturates at $0.688$.

\begin{figure}[t]
    \centering
    \includegraphics[width=1.0\linewidth]{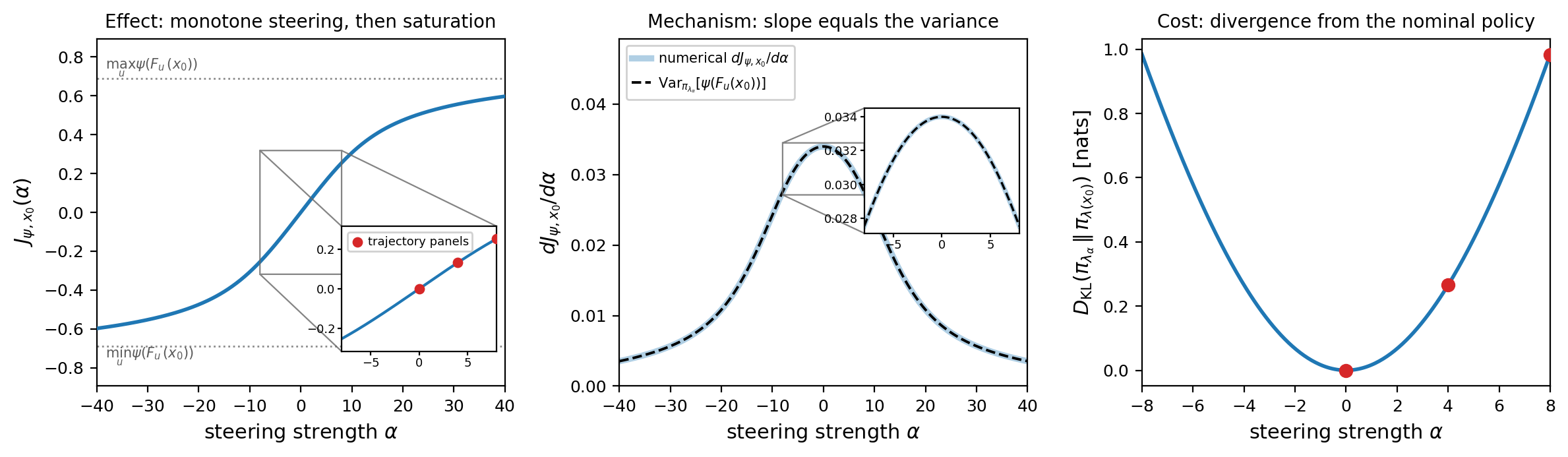}
    \caption{\textbf{Effect, mechanism, and cost of linear steering.}
    \emph{Left:} $J_{\psi,x_0}(\alpha)$; dotted lines mark the extreme values of $\psi(F_\mathbf{u}(x_0))$ over the candidate set.
    \emph{Middle:} numerical $\mathrm{d}J_{\psi,x_0}/\mathrm{d}\alpha$ against
    $\mathrm{Var}_{\pi_{\lambda_\alpha}}[\psi(F_\mathbf{u}(x_0))]$.
    \emph{Right:} $D_{\mathrm{KL}}(\pi_{\lambda_\alpha}\|\pi_{\lambda(x_0)})$.
    Insets and the right panel show $\alpha\in[-8,8]$; red markers are 
    $\alpha\in\{0,4,8\}$.}
    \label{fig:oracle_slope}
\end{figure}

\subsection{Probing and Steering a Pretrained VLA}
\label{sec:vla_main}

Similarly to the oracle setting, we ask two questions on OpenVLA-OFT. \textbf{Q1:} does the representation of a pretrained VLA encode, linearly, what an action chunk does to a physical quantity (\Cref{def:linear_probing})? \textbf{Q2:} can the probe direction steer the policy (\Cref{def:linear_steering})? Details are in \Cref{sec:oft_experiments}.

\paragraph{Setup.}
OpenVLA-OFT maps its last-layer action-token states $\lambda(x)\in\mathbb R^{56\times4096}$ through an MLP head to a chunk of $H_{\mathrm{chunk}}=8$ actions, $\mathbf u(x)=\mathrm{head}(\lambda(x))$; To retain the most informative components, we use the latent representation $z(x)=(1,P\lambda(x))$, where $P$ denotes the PCA projection matrix. On a LIBERO-10 task, $\ell =$ \emph{``put the white mug on the left plate and put the yellow-and-white mug on the right plate''}, we collect 788 states from 50 layouts and, at 404 of them, execute 32 candidate chunks. The QoIs are the mug positions (\(\mathrm{mug}_1\): white; \(\mathrm{mug}_2\): yellow-and-white), the end-effector (EE) height as a control, and the task-relevant negative EE-to-mug distances.

\paragraph{Q1: action-conditioned linear probing.}
We truncate the signature at level two, $S_{\le2}(\mathbf u)=(S^\emptyset, S_1, S_2)$ with $S^\emptyset=1$, and split the probe of \Cref{def:linear_probing} along these coordinates into a static part, paired with the empty word, and an action-conditioned increment, paired with the words of levels one and two denoted by $S_{1:2}$,
\[
  \hat\psi\big(F_{\mathbf u}(x)\big)
  =\langle z(x), L_\psi S_{\le 2}(\mathbf{u})\rangle_{\mathbb R^q}= \underbrace{\langle z(x),\, L^\emptyset_\psi\rangle_{\mathbb R^q}}_{\text{static } \psi(x)}
  + \underbrace{\big\langle z(x),\, L^{1:2}_\psi\, S_{1:2}(\mathbf u)\big\rangle_{\mathbb R^q}}_{\text{increment } \Delta\psi(x;\mathbf u)} ,
  \label{eq:vla_probe}
\]
and compare the bilinear increment $\langle z, L_{\psi}^{1:2}S_{1:2} \rangle$ with an additive probe $\langle z,a\rangle+\langle b, S_{1:2}\rangle$, in which a chunk has the same effect in every state (\Cref{tab:vla_probe}). All results are on held-out episodes. The bilinear increment predicts how a candidate chunk moves the mugs on held-out episodes ($R^2$ $0.90$--$0.98$, against $0.80$--$0.95$ for the additive probe), and nonlinear readouts add little ($\le0.03$). The static probe, the standard tool of prior probing work, reads the current mug positions well but cannot predict them after the chunk ($0.64$ and $0.62$ for the mug heights, against $0.95$ and $0.96$ with the increment). The representation therefore encodes more than the current state: it encodes, linearly and depending on the state, what an action will do to the objects.


\begin{table}[t]
\centering
\caption{\textbf{Q1: probing OpenVLA-OFT.} Held-out episode $R^2$ for predicting the current QoI $\psi(x)$, chunk-induced increment $\Delta\psi$, and propagated QoI $\psi(F_{\mathbf u}(x))$. Additional baselines: \Cref{tab:oft_propagated,tab:oft_increment}.}
\label{tab:vla_probe}
\setlength{\tabcolsep}{5pt}
\begin{tabular}{l|c|c>{\columncolor{gray!10}}c|c>{\columncolor{gray!10}}c}
\hline
 & static
 & \multicolumn{2}{c|}{increment $\Delta\psi$}
 & \multicolumn{2}{c}{propagated $\psi(F_{\mathbf u}(x))$} \\
QoI $\psi$
 & $\psi(x)$
 & additive
 & \textbf{bilinear}
 & static only
 & \textbf{static $+$ bilinear} \\
\hline
mug$_1$ $p_x$ & .777 & .933 & .960 & .671 & .767 \\
mug$_1$ $p_y$ & .865 & .895 & .959 & .821 & .861 \\
mug$_1$ $p_z$ & .956 & .877 & .918 & .636 & .946 \\
mug$_2$ $p_x$ & .580 & .801 & .901 & .453 & .593 \\
mug$_2$ $p_y$ & .814 & .953 & .979 & .752 & .846 \\
mug$_2$ $p_z$ & .965 & .864 & .923 & .620 & .959 \\
\hline
\end{tabular}
\end{table}

\paragraph{Q2: steering along the probe direction.}
We keep the physical state fixed, move the representation that the head receives along a latent direction, $z_\alpha=z(x)+\alpha\,z_\psi(x)$, and measure $J_{\psi, x}(\alpha)=\psi(F_{\mathbf u(z_\alpha)}(x))$. Unlike \Cref{eq:linear-path}, the quantity $\lambda_\psi(x)=L_\psi^\top z(x)$ here lives in signature coordinates, not in the latent space of $z$, and therefore cannot be added to $z$ directly. Instead, the probe predicts $\hat J_{\psi, x}(\alpha) = \hat{\psi}(F_{\mathbf{u}(z_\alpha)}(x)) = \left\langle \lambda_\psi(x), S_{\leq2}\bigl(\mathbf u(z_\alpha)\bigr) \right\rangle_{\mathbb R^q},$
and the corresponding latent steering direction is obtained by applying the
chain rule through the action head.
\[
  \frac{\mathrm d\hat J_{\psi, x}}{\mathrm d\alpha}\Big|_{\alpha=0}
  = \Big\langle \lambda_\psi(x),\, \frac{\mathrm d}{\mathrm d\alpha} S_{\le2}\big(\mathbf u(z_\alpha)\big)\Big|_{\alpha=0} \Big\rangle_{\mathbb R^q}=\big\langle D\phi\big(z(x)\big)^{\!\top} \lambda_\psi(x),\, z_\psi(x)\big\rangle_{\mathbb R^q},
  \label{eq:vla_chain}
\]
where \(D\phi(z):=\partial S_{\le 2}(\mathbf{u}(z))/\partial z\). We therefore steer along $z_\psi=D\phi^\top\lambda_\psi$, the gradient of the probe-predicted QoI with respect to the latent, which requires only the probe and the model's own head. We compare it with the probe direction without the head, $\gamma_\psi=L^{1:2}_\psi S_{1:2}(\mathbf u)$, the static direction $L^\emptyset_\psi$, random directions, and the simulator gradient $\nabla_z J_{\psi, x}$ as an oracle (\Cref{tab:vla_steer}).

Our results show that every held-out state admits a linear steering direction, which the probe obtains. $D\phi^\top\lambda_\psi$ increases the QoI monotonically at 138 of 140 states and agrees with the simulator gradient (cosine $0.98$--$1.00$) for all seven QoIs (\Cref{tab:vla_steer}). The two directions that ignore how the policy reads its representation fail: the probe direction taken without the head, $\gamma_\psi$, as in probe steering of language models, recovers at most three quarters of the effect, and the static direction often moves the quantity the wrong way. Being able to decode a quantity from the representation is thus not the same as being able to control it; control requires the action-conditioned probe together with the model's own action map, which here plays the role that the identification $D\phi=\nabla^2A$ plays in \Cref{thm:linear_steering}. Since both are available from the model, the direction is obtained without a simulator.

In closed loop, steering is applied at every re-plan, recomputing $z$, $\lambda_\psi$, and $D\phi$ at the current state (\Cref{fig:vla_closed_loop}).  The effect accumulates in six times ($2.4$\,s). For $\psi=-d(\mathrm{EE},\mathrm{mug}_j)$, the distance to $\mathrm{mug}_1$ (white mug) after $2.4$\,s is $20/11.8/9$\,cm for $\alpha=-16/0/16$, and with the probe of $\mathrm{mug}_2$ (yellow-and-white mug), the distance to that mug is $29.8/26.9/24.1$\,cm. Which distance changes most depends on the mug that defines $\psi$, while the policy keeps its task: it still heads for the white mug, with its path biased by a few centimetres towards or away from the chosen object.

\begin{table}[t]
\centering
\caption{\textbf{Q2: steering OpenVLA-OFT.}
Across 20 held-out states per QoI, we report the median effect
$J_{\psi,x}(8)-J_{\psi,x}(-8)$ in mm at equal representation
displacement (random: median over three directions).
For $D\phi^\top\lambda_\psi$, we additionally report the
percentage of monotone states and median cosine similarity
to the oracle simulator gradient. Full results: \Cref{tab:oft_steering,fig:oft_curves}.}
\label{tab:vla_steer}
\setlength{\tabcolsep}{4.5pt}
\begin{tabular}{l|c>{\columncolor{gray!10}}cccc|cc}
\hline
 & \multicolumn{5}{c|}{effect (mm)}
 & \multicolumn{2}{c}
 {\rule{0pt}{11pt}$D\phi^\top\lambda_\psi$} \\
 & sim.\ gradient
 & $\boldsymbol{D\phi^\top\lambda_\psi}$
 & $\gamma_\psi$
 & static
 & random
 & monotone
 & cos \\
\hline
EE $p_z$ (control)          & $+9.5$  & $+9.8$  & $+0.4$ & $-4.2$  & $+0.6$ & 100\% & 1.00 \\
mug$_1$ $p_z$              & $+12.9$ & $+12.9$ & $+2.1$ & $-2.2$  & $-0.6$ & 100\% & 1.00 \\
mug$_2$ $p_z$              & $+11.6$ & $+11.2$ & $+5.5$ & $-2.1$  & $-1.0$ & 100\% & 1.00 \\
$-d$(mug$_1$, plate)       & $+12.3$ & $+11.9$ & $+7.6$ & $-11.3$ & $-1.8$ & 100\% & 0.98 \\
$-d$(mug$_2$, plate)       & $+12.4$ & $+12.2$ & $+9.2$ & $+3.5$  & $-0.2$ & 100\% & 0.99 \\
$-d$(EE, mug$_1$)          & $+9.0$  & $+8.9$  & $+5.3$ & $-0.0$  & $+0.2$ & 95\%  & 1.00 \\
$-d$(EE, mug$_2$)          & $+9.4$  & $+9.3$  & $+1.7$ & $+2.9$  & $-0.4$ & 95\%  & 1.00 \\
\hline
\end{tabular}
\end{table}

\section{Conclusion}

We developed a control-theoretic formulation of the linear representation hypothesis for vision-language-action models, centered on the evolution of physical quantities under candidate actions. On the representation side, we showed that propagated QoIs admit action-conditioned linear probes through signature-based measurement maps. On the policy side, we introduced a signature generalized linear model and showed that this geometry induces linear steering directions along which the expected propagated QoI changes monotonically.

Our results support the proposed linear probing and steering framework at both oracle and pretrained-VLA levels. An explicit oracle construction verifies the predicted mechanisms exactly, including the derivative--variance identity, while experiments on OpenVLA-OFT provide evidence that analogous structure can emerge in a learned VLA. Given the still-emerging VLA landscape, however, identifying a representation $\lambda$ naturally suited to both probing and steering remains an open challenge. We therefore view our experiments primarily as a proof of concept for the existence and realizability of such a representation.

\begin{figure}[h]
\centering
\includegraphics[width=0.95\linewidth]{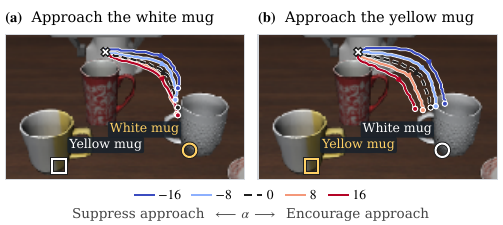}
\caption{\textbf{Closed-loop concept steering.} Natural-language labels \emph{``Approach \(\mathit{mug}_j\)''} are used only to describe the corresponding concepts, not as additional policy instructions \(\ell \mapsto \ell'\). End-effector trajectories over six re-planned chunks (2.4\,s) in a held-out episode, starting from the same state ($\times$), with $\alpha D\phi^\top\lambda_\psi$ applied at each re-plan. The QoI is $\psi(x) = -d(\mathrm{EE},\mathrm{mug}_j)$, targeting (a) the white mug ($j=1$) and (b) the yellow-and-white mug ($j=2$). Positive and negative $\alpha$ encourage and suppress approach, respectively. In both cases the policy still heads for the white mug, its first subgoal; steering biases the path and pace towards or away from the targeted mug without changing the object the policy goes to.}
\label{fig:vla_closed_loop}
\end{figure}

\section*{Policy on LLM Use for Research and Writing Papers}
We used generative AI tools to assist with language polishing and to discuss and refine aspects of the theoretical formulation. All mathematical statements and final formulations were independently verified by the authors.

\bibliography{iclr2027_conference}
\bibliographystyle{iclr2027_conference}

\newpage
\appendix

\section{Regularity Conditions for the Chen--Fliess Expansion}
\label{app:chen-fliess}

We collect the notation and sufficient conditions used for the
Chen--Fliess expansion in \Cref{thm:existence}. Since the approximation
result is local in the state variable, we formulate the conditions
uniformly over an arbitrary compact set \(K\subset\mathcal{X}\) and the
admissible action class \(\mathcal{U}_H\). The regularity and convergence
constants below may depend on \(K\), but are uniform over
\(K\times\mathcal{U}_H\).

\paragraph{Controlled system.}
Consider the control-affine system
\begin{equation}
\dot x(t)
=
f_0(x(t))
+
\sum_{i=1}^m f_i(x(t))u^i(t),
\label{eq:appendix-controlled-system}
\end{equation}
and define $u^0(t)\equiv 1$. Then
\[
\dot x(t)
=
\sum_{i=0}^m f_i(x(t))u^i(t).
\]
Let $\mathcal{A}=\{0,1,\ldots,m\}$ denote the input alphabet, and let \(\mathcal{A}^\star\) denote the set
of all finite words over \(\mathcal{A}\), including the empty word
\(\emptyset\).

For a word $I =(i_1,\ldots,i_k)\in\mathcal{A}^\star$, we use the same chronological convention as in \Cref{def:signature}.
Namely,
\begin{equation}
S(\mathbf{u})^I_{0,H}
:=
\int_{0\le t_1\le\cdots\le t_k\le H}
\prod_{r=1}^k u^{i_r}(t_r)
\,\mathrm{d}t_1\cdots \mathrm{d}t_k,
\qquad
S(\mathbf{u})^\emptyset_{0,H}=1.
\label{eq:appendix-signature}
\end{equation}
With this convention, define the corresponding iterated Lie derivative
by
\begin{equation}
\mathcal{L}_I\psi
=
\mathcal{L}_{f_{i_1}}\cdots\mathcal{L}_{f_{i_k}}\psi,
\qquad
\mathcal{L}_\emptyset\psi=\psi,
\label{eq:appendix-lie-word}
\end{equation}
where $\mathcal{L}_{f_i}\varphi(x)=D\varphi(x)f_i(x)$. The ordering in
\cref{eq:appendix-signature,eq:appendix-lie-word}
is chosen so that the Chen--Fliess expansion takes the form
\[
\psi(F_{\mathbf{u}}(x))
=
\sum_{I\in\mathcal{A}^\star}
\mathcal{L}_I\psi(x)S(\mathbf{u})^I_{0,H}.
\]

\subsection{Sufficient conditions for locally uniform convergence}

Fix an arbitrary compact set $K\subset\mathcal{X}$. We impose the following sufficient conditions.

\begin{assumption}[Common reachable domain]
\label{ass:reachable-domain}
There exist an open set \(D\subset\mathbb{R}^n\) and a compact set
\(K_H\subset D\) such that, for every \(x\in K\) and every
\(\mathbf{u}\in\mathcal{U}_H\), the controlled system admits a unique
solution on \([0,H]\), and
\[
x_{\mathbf{u}}(t;x)\in K_H,
\qquad
\forall t\in[0,H].
\]
Equivalently, the horizon-\(H\) reachable tube from \(K\),
\[
\mathcal{R}_H(K)
:=
\left\{
x_{\mathbf{u}}(t;x):
x\in K,\,
\mathbf{u}\in\mathcal{U}_H,\,
t\in[0,H]
\right\},
\]
satisfies $\mathcal{R}_H(K)\subset K_H\subset D$.
\end{assumption}

\begin{assumption}[Analytic regularity]
\label{ass:analytic-regularity}
The vector fields $f_0,\ldots,f_m:D\to\mathbb{R}^n$ and each QoI $\psi_j:D\to\mathbb{R},
\, j=1,\ldots,p$, are real analytic on \(D\).
\end{assumption}

\begin{assumption}[Uniformly bounded controls]
\label{ass:bounded-controls}
Every admissible control satisfies
\[
u^i\in L^\infty([0,H]),
\qquad
i=1,\ldots,m,
\]
and the admissible control class is uniformly bounded:
\begin{equation}
\bar u
:=
\sup_{\mathbf{u}\in\mathcal{U}_H}
\max\left\{
1,
\|u^1\|_{L^\infty},
\ldots,
\|u^m\|_{L^\infty}
\right\}
<\infty.
\label{eq:uniform-control-bound}
\end{equation}
\end{assumption}

\begin{assumption}[Uniform Lie-derivative growth]
\label{ass:lie-growth}
There exist constants \(C_K,M_K>0\), possibly depending on \(K\) but
independent of \(j\), \(x\), and \(I\), such that
\begin{equation}
\left|
\mathcal{L}_I\psi_j(x)
\right|
\le
C_K M_K^{|I|}|I|!,
\qquad
\forall j\in\{1,\ldots,p\},
\quad
\forall x\in K_H,
\quad
\forall I\in\mathcal{A}^\star.
\label{eq:lie-growth}
\end{equation}
\end{assumption}

\begin{assumption}[Local convergence radius]
\label{ass:cf-smallness}
The horizon and admissible control magnitude satisfy
\begin{equation}
q_K
:=
(m+1)M_K\bar uH
<
1.
\label{eq:cf-smallness}
\end{equation}
\end{assumption}

\begin{lemma}[Uniform bound on signature coordinates]
\label{lem:signature-bound}
Under \cref{ass:bounded-controls}, for every
\(\mathbf{u}\in\mathcal{U}_H\) and every
\(I\in\mathcal{A}^\star\),
\begin{equation}
|S(\mathbf{u})^I_{0,H}|
\le
\frac{(\bar uH)^{|I|}}{|I|!}.
\label{eq:signature-bound}
\end{equation}
\end{lemma}

\begin{proof}
Let \(I=(i_1,\ldots,i_k)\), so that \(|I|=k\). By
\cref{eq:uniform-control-bound},
\[
|u^{i_r}(t)|\le\bar u,
\qquad
r=1,\ldots,k,
\]
where the same bound holds for \(i_r=0\) because \(u^0\equiv1\) and
\(\bar u\ge1\). Hence
\[
\begin{aligned}
|S(\mathbf{u})^I_{0,H}|
&\le
\int_{0\le t_1\le\cdots\le t_k\le H}
\prod_{r=1}^k |u^{i_r}(t_r)|
\,\mathrm{d}t_1\cdots \mathrm{d}t_k \\
&\le
\bar u^k
\operatorname{Vol}
\left\{
0\le t_1\le\cdots\le t_k\le H
\right\}.
\end{aligned}
\]
Since the ordered \(k\)-simplex has volume \(H^k/k!\),
\[
|S(\mathbf{u})^I_{0,H}|
\le
\bar u^k\frac{H^k}{k!}
=
\frac{(\bar uH)^k}{k!}.
\]
The empty-word case follows from \(S(\mathbf{u})^\emptyset_{0,H}=1\).
\end{proof}

\begin{proposition}[Locally uniform Chen--Fliess expansion]
\label{prop:uniform-chen-fliess}
Let \(K\subset\mathcal{X}\) be compact. Under
\cref{ass:reachable-domain,ass:analytic-regularity,ass:bounded-controls,ass:lie-growth,ass:cf-smallness},
for every \(j=1,\ldots,p\),
\begin{equation}
\psi_j(F_{\mathbf{u}}(x))
=
\sum_{I\in\mathcal{A}^\star}
\mathcal{L}_I\psi_j(x)S(\mathbf{u})^I_{0,H},
\label{eq:uniform-cf-series}
\end{equation}
and the series converges absolutely and uniformly over
\[
j\in\{1,\ldots,p\},
\qquad
x\in K,
\qquad
\mathbf{u}\in\mathcal{U}_H.
\]
Moreover, if
\[
R_{j,N}(x,\mathbf{u})
:=
\sum_{|I|>N}
\mathcal{L}_I\psi_j(x)S(\mathbf{u})^I_{0,H},
\]
then
\begin{equation}
\max_{1\le j\le p}
\sup_{(x,\mathbf{u})\in K\times\mathcal{U}_H}
|R_{j,N}(x,\mathbf{u})|
\le
\frac{C_Kq_K^{N+1}}{1-q_K},
\qquad
q_K=(m+1)M_K\bar uH<1.
\label{eq:uniform-cf-remainder}
\end{equation}
\end{proposition}

\begin{proof}
Fix \(j\in\{1,\ldots,p\}\), \(x\in K\), and
\(\mathbf{u}\in\mathcal{U}_H\). For any sufficiently smooth scalar
function \(g:D\to\mathbb{R}\), the fundamental theorem of calculus along
the controlled trajectory gives
\[
\begin{aligned}
g(x_{\mathbf{u}}(t;x))-g(x)
&=
\int_0^t
Dg(x_{\mathbf{u}}(s;x))
\dot x_{\mathbf{u}}(s;x)
\,\mathrm{d}s \\
&=
\sum_{i=0}^m
\int_0^t
\mathcal{L}_{f_i}g
\bigl(x_{\mathbf{u}}(s;x)\bigr)
u^i(s)\,\mathrm{d}s.
\end{aligned}
\]
Repeated application of this identity yields, for every \(N\ge0\),
\begin{equation}
\psi_j(F_{\mathbf{u}}(x))
=
\sum_{|I|\le N}
\mathcal{L}_I\psi_j(x)S(\mathbf{u})^I_{0,H}
+
\widetilde R_{j,N}(x,\mathbf{u}),
\label{eq:finite-cf}
\end{equation}
where \(\widetilde R_{j,N}\) is a sum of iterated integrals of order
\(N+1\). Each such term contains an iterated Lie derivative of order
\(N+1\), evaluated at an intermediate point of the controlled
trajectory.

By \cref{ass:reachable-domain}, every such intermediate point belongs
to \(K_H\). Hence \cref{ass:lie-growth} gives
\[
\left|
\mathcal{L}_I\psi_j
\bigl(x_{\mathbf{u}}(t;x)\bigr)
\right|
\le
C_KM_K^{N+1}(N+1)!
\]
for every word \(w\) of length \(N+1\). Using the uniform control bound
and the volume of the ordered simplex, each remainder term is bounded
by
\[
C_KM_K^{N+1}(N+1)!
\frac{(\bar uH)^{N+1}}{(N+1)!}
=
C_K(M_K\bar uH)^{N+1}.
\]
Since there are \((m+1)^{N+1}\) words of length \(N+1\),
\[
|\widetilde R_{j,N}(x,\mathbf{u})|
\le
C_Kq_K^{N+1}.
\]
Since \(q_K<1\), this converges to zero uniformly over
\[
j\in\{1,\ldots,p\},
\qquad
x\in K,
\qquad
\mathbf{u}\in\mathcal{U}_H.
\]
Letting \(N\to\infty\) in \cref{eq:finite-cf} therefore yields
\cref{eq:uniform-cf-series} uniformly on
\(K\times\mathcal{U}_H\).

For absolute convergence, let \(|I|=k\). By
\Cref{ass:lie-growth,lem:signature-bound},
\[
\begin{aligned}
\left|
\mathcal{L}_I\psi_j(x)S(\mathbf{u})^I_{0,H}
\right|
&\le
C_KM_K^k k!
\frac{(\bar uH)^k}{k!} \\
&=
C_K(M_K\bar uH)^k.
\end{aligned}
\]
There are \((m+1)^k\) words of length \(k\). Therefore
\[
\sum_{|I|=k}
\left|
\mathcal{L}_I\psi_j(x)S(\mathbf{u})^I_{0,H}
\right|
\le
C_Kq_K^k.
\]
Since \(q_K<1\),
\[
\sum_{k=0}^\infty C_Kq_K^k
=
\frac{C_K}{1-q_K}
<
\infty.
\]
The bound is independent of \(j\), \(x\in K\), and
\(\mathbf{u}\in\mathcal{U}_H\). Hence the Weierstrass \(M\)-test gives
absolute and uniform convergence.

Finally,
\[
\begin{aligned}
|R_{j,N}(x,\mathbf{u})|
&\le
\sum_{k=N+1}^\infty
\sum_{|I|=k}
\left|
\mathcal{L}_I\psi_j(x)S(\mathbf{u})^I_{0,H}
\right| \\
&\le
C_K\sum_{k=N+1}^\infty q_K^k \\
&=
\frac{C_Kq_K^{N+1}}{1-q_K}.
\end{aligned}
\]
Taking the maximum over \(j\) and the supremum over
\(K\times\mathcal{U}_H\) proves \cref{eq:uniform-cf-remainder}.
\end{proof}

\newpage

\section{Proofs}
\subsection{Proof of~\Cref{thm:existence}}
\label[appendixsubsection]{pf:existence}



\begin{proof}
Fix an arbitrary compact set \(K\subset\mathcal{X}\) and
\(\varepsilon>0\). Let $\mathcal{A}=\{0,1,\ldots,m\}$, where \(u^0(t)\equiv1\), and let \(\mathcal{A}^\star\) denote the set of
all finite words over \(\mathcal{A}\), including the empty word
\(\emptyset\).

By the locally uniform Chen--Fliess expansion established in
\cref{app:chen-fliess}, for every \(j\in\{1,\ldots,p\}\),
\begin{equation}
\psi_j\bigl(F_{\mathbf{u}}(x)\bigr)
=
\sum_{I\in\mathcal{A}^\star}
\mathcal{L}_I\psi_j(x)S(\mathbf{u})^I_{0,H},
\label{eq:cf-expansion-proof}
\end{equation}
where the series converges absolutely and uniformly over
\[
(x,\mathbf{u})\in K\times\mathcal{U}_H.
\]

More precisely, under the conditions of
\cref{app:chen-fliess}, there are exactly \((m+1)^k\) words of length \(k\),
\[
\begin{aligned}
\sum_{\substack{I\in\mathcal{A}^\star\\|I|=k}}
\left|
\mathcal{L}_I\psi_j(x)S(\mathbf{u})^I_{0,H}
\right|
&\le
(m+1)^kC_KM_K^kk!
\frac{(\bar uH)^k}{k!} \\
&=
C_Kq_K^k.
\end{aligned}
\]

Consequently, if
\[
R_{j,N}(x,\mathbf{u})
:=
\sum_{\substack{I\in\mathcal{A}^\star\\|I|>N}}
\mathcal{L}_I\psi_j(x)S(\mathbf{u})^I_{0,H}
\]
denotes the truncation remainder, then
\begin{equation}
\max_{1\le j\le p}
\sup_{(x,\mathbf{u})\in K\times\mathcal{U}_H}
|R_{j,N}(x,\mathbf{u})|
\le
\frac{C_Kq_K^{N+1}}{1-q_K}.
\label{eq:cf-uniform-tail}
\end{equation}

Since \(q_K<1\), the right-hand side of
\cref{eq:cf-uniform-tail} converges to zero as \(N\to\infty\).
Therefore, there exists \(N\in\mathbb{N}\) such that
\begin{equation}
\frac{C_Kq_K^{N+1}}{1-q_K}
<
\varepsilon.
\label{eq:choose-N}
\end{equation}

Define the finite word set
\[
\mathcal{W}_N
:=
\mathcal{A}_{\le N}^\star
=
\left\{
I\in\mathcal{A}^\star:
|I|\le N
\right\},
\]
and let
\[
d_N
:=
|\mathcal{W}_N|
=
\sum_{k=0}^N(m+1)^k.
\]
Fix an ordering
\[
\mathcal{W}_N
=
\{I_1,\ldots,I_{d_N}\}.
\]

For each QoI \(\psi_j\), define the truncated Chen--Fliess coefficient
vector
\[
c_{j,N}(x)
:=
\begin{bmatrix}
\mathcal{L}_{I_1}\psi_j(x)\\
\vdots\\
\mathcal{L}_{I_{d_N}}\psi_j(x)
\end{bmatrix}
\in\mathbb{R}^{d_N}.
\]
Construct the shared representation
\begin{equation}
\lambda_N(x)
:=
\begin{bmatrix}
c_{1,N}(x)\\
\vdots\\
c_{p,N}(x)
\end{bmatrix}
\in\mathbb{R}^{pd_N}.
\label{eq:oracle-shared-rep}
\end{equation}

Define the coordinate projection
\[
P_N:\mathcal{S}\to\mathbb{R}^{d_N}
\]
by
\[
P_Ns
:=
\begin{bmatrix}
s^{I_1}\\
\vdots\\
s^{I_{d_N}}
\end{bmatrix}.
\]
In particular,
\[
P_NS(\mathbf{u})
=
\begin{bmatrix}
S(\mathbf{u})^{I_1}_{0,H}\\
\vdots\\
S(\mathbf{u})^{I_{d_N}}_{0,H}
\end{bmatrix}.
\]

For each \(j\in\{1,\ldots,p\}\), let
\[
E_j:\mathbb{R}^{d_N}\to\mathbb{R}^{pd_N}
\]
denote the canonical embedding into the \(j\)-th block:
\[
E_jz
=
\begin{bmatrix}
0\\
\vdots\\
z\\
\vdots\\
0
\end{bmatrix}.
\]
Define
\begin{equation}
L_{\psi_j,N}
:=
E_j\circ P_N:
\mathcal{S}\to\mathbb{R}^{pd_N}.
\label{eq:B-construction}
\end{equation}
Since both \(P_N\) and \(E_j\) are linear,
\(L_{\psi_j,N}\) is linear.

The corresponding action-conditioned measurement map is

$$
\gamma_{\psi_j,N}(\mathbf{u})
:=
L_{\psi_j,N}S(\mathbf{u})
=
E_jP_NS(\mathbf{u}).
$$

Using the standard Euclidean inner products on the corresponding spaces,

\[
\begin{aligned}
\left\langle
\lambda_N(x),
\gamma_{\psi_j,N}(\mathbf{u})
\right\rangle_{\mathbb{R}^{p d_N}}
&=
\left\langle
c_{j,N}(x),
P_N S(\mathbf{u})
\right\rangle_{\mathbb{R}^{d_N}} \\
&=
\sum_{r=1}^{d_N}
\mathcal{L}_{I_r}\psi_j(x)\,
S(\mathbf{u})^{I_r}_{0,H} \\
&=
\sum_{\substack{I\in\mathcal{A}^{\star} \\ |I|\le N}}
\mathcal{L}_I\psi_j(x)\,
S(\mathbf{u})^I_{0,H}.
\end{aligned}
\]

Combining this identity with \cref{eq:cf-expansion-proof} gives
\[
\psi_j(F_{\mathbf{u}}(x))
-
\left\langle
\lambda_N(x),
\gamma_{\psi_j,N}(\mathbf{u})
\right\rangle_{\mathbb{R}^{pd_N}}
=
R_{j,N}(x,\mathbf{u}).
\]
Hence, by \cref{eq:choose-N},
\[
\begin{aligned}
\max_{1\le j\le p}
\sup_{(x,\mathbf{u})\in K\times\mathcal{U}_H}
\left|
\psi_j(F_{\mathbf{u}}(x))
-
\left\langle
\lambda_N(x),
\gamma_{\psi_j,N}(\mathbf{u})
\right\rangle_{\mathbb{R}^{pd_N}}
\right|
&=
\max_{1\le j\le p}
\sup_{(x,\mathbf{u})\in K\times\mathcal{U}_H}
|R_{j,N}(x,\mathbf{u})| \\
&\le
\frac{C_Kq_K^{N+1}}{1-q_K} \\
&<
\varepsilon.
\end{aligned}
\]
Finally, choose
\[
d:=pd_N,
\qquad
\lambda:=\lambda_N,
\qquad
L_{\psi_j}:=L_{\psi_j,N}.
\]
Then
\[
\gamma_{\psi_j}
=
L_{\psi_j}\circ S:
\mathcal{U}_H\to\mathbb{R}^d
\]
is linear in the signature coordinates and satisfies
\[
\max_{1\le j\le p}
\sup_{(x,\mathbf{u})\in K\times\mathcal{U}_H}
\left|
\psi_j(F_{\mathbf{u}}(x))
-
\left\langle
\lambda(x),
\gamma_{\psi_j}(\mathbf{u})
\right\rangle_{\mathbb{R}^d}
\right|
<
\varepsilon.
\]
Since \(K\subset\mathcal{X}\) was arbitrary, the approximation holds
locally uniformly in the state variable.
\end{proof}

\subsection{Proof of~\Cref{thm:linear_steering}}
\label[appendixsubsection]{pf:linear_steering}
\begin{proof}
Fix \(x\in\mathcal{X}\), and write
\[
\lambda_0=\lambda(x),
\qquad
\lambda_\psi=\lambda_\psi(x)=L_\psi^*\lambda_0.
\]
We consider values of \(\alpha\) for which
\[
\lambda_\alpha
=
\lambda_0+\alpha\lambda_\psi
\in\Theta,
\]
so that \(\pi_{\lambda_\alpha}(\cdot\mid x)\) is well defined.

By definition of the adjoint \(L_\psi^*\),
\[
\left\langle
\lambda_0,
L_\psi S(\mathbf{u})
\right\rangle_{\mathcal{H}_S}
=
\left\langle
L_\psi^*\lambda_0,
S(\mathbf{u})
\right\rangle_{\mathcal{H}_S}
=
\left\langle
\lambda_\psi,
S(\mathbf{u})
\right\rangle_{\mathcal{H}_S}.
\]
Since \(\psi\) admits an exact linear probe at \(x\),
\[
\psi\bigl(F_{\mathbf{u}}(x)\bigr)
=
\left\langle
\lambda_0,
L_\psi S(\mathbf{u})
\right\rangle_{\mathcal{H}_S},
\qquad
\forall\,\mathbf{u}\in\mathcal{U}_H,
\]
and therefore
\begin{equation}
\psi\bigl(F_{\mathbf{u}}(x)\bigr)
=
\left\langle
\lambda_\psi,
S(\mathbf{u})
\right\rangle_{\mathcal{H}_S}.
\label{eq:steering-statistic-identity}
\end{equation}
Thus the propagated QoI coincides exactly with the sufficient-statistic
direction induced by \(\lambda_\psi\).

From \cref{eq:steering-statistic-identity},
\[
\begin{aligned}
J_{\psi,x}(\alpha)
&=
\mathbb{E}_{\mathbf{u}\sim
\pi_{\lambda_\alpha}(\cdot\mid x)}
\left[
\left\langle
\lambda_\psi,
S(\mathbf{u})
\right\rangle_{\mathcal{H}_S}
\right]
\\
&=
\left\langle
\lambda_\psi,
\mathbb{E}_{\mathbf{u}\sim
\pi_{\lambda_\alpha}(\cdot\mid x)}
\left[
S(\mathbf{u})
\right]
\right\rangle_{\mathcal{H}_S}.
\end{aligned}
\]
For the signature exponential family, the dual coordinate is
\[
\varphi(\lambda)
= \nabla A(\lambda) = \mathbb{E}_{\mathbf{u}\sim
\pi_{\lambda}(\cdot\mid x)}
\left[
S(\mathbf{u})
\right].
\]
Hence
\begin{equation}
J_{\psi,x}(\alpha)
=
\left\langle
\lambda_\psi,
\varphi(\lambda_\alpha)
\right\rangle_{\mathcal{H}_S}.
\label{eq:J-dual-coordinate}
\end{equation}

Since the steering path is affine, i.e.,
$\lambda_\alpha
=
\lambda_0+\alpha\lambda_\psi$,
its derivative with respect to \(\alpha\) is
\[
\frac{\mathrm{d}}{\mathrm{d}\alpha}\lambda_\alpha
=
\lambda_\psi.
\]
Differentiating \cref{eq:J-dual-coordinate} along this path and applying
the chain rule gives
\[
\begin{aligned}
\frac{\mathrm{d}}{\mathrm{d}\alpha}
J_{\psi,x}(\alpha)
&=
\left\langle
\lambda_\psi,
D\varphi(\lambda_\alpha)
\left[
\frac{\mathrm{d}\lambda_\alpha}{\mathrm{d}\alpha}
\right]
\right\rangle_{\mathcal{H}_S}
\\
&=
\left\langle
\lambda_\psi,
D\varphi(\lambda_\alpha)
\left[
\lambda_\psi
\right]
\right\rangle_{\mathcal{H}_S}.
\end{aligned}
\]
Because $\varphi(\lambda)=\nabla A(\lambda)$,
we have
$D\varphi(\lambda)
=
\nabla^2 A(\lambda)$,
and therefore
\begin{equation}
\frac{\mathrm{d}}{\mathrm{d}\alpha}
J_{\psi,x}(\alpha)
=
\left\langle
\lambda_\psi,
\nabla^2 A(\lambda_\alpha)
[\lambda_\psi]
\right\rangle_{\mathcal{H}_S}.
\label{eq:J-hessian}
\end{equation}

It remains to identify the quadratic form in
\cref{eq:J-hessian}. For the exponential family,
the Hessian of the log-partition function is the covariance operator of
the sufficient statistic. More precisely, for any
\(\eta,\zeta\in\mathcal{H}_S\),
\[
\left\langle
\eta,
\nabla^2 A(\lambda)[\zeta]
\right\rangle_{\mathcal{H}_S}
=
\operatorname{Cov}_{\mathbf{u}\sim
\pi_{\lambda}(\cdot\mid x)}
\left(
\left\langle
\eta,S(\mathbf{u})
\right\rangle,
\left\langle
\zeta,S(\mathbf{u})
\right\rangle_{\mathcal{H}_S}
\right).
\]
Taking
\[
\eta=\zeta=\lambda_\psi
\]
yields
\[
\left\langle
\lambda_\psi,
\nabla^2 A(\lambda_\alpha)[\lambda_\psi]
\right\rangle_{\mathcal{H}_S}
=
\operatorname{Var}_{\mathbf{u}\sim
\pi_{\lambda_\alpha}(\cdot\mid x)}
\left[
\left\langle
\lambda_\psi,
S(\mathbf{u})
\right\rangle_{\mathcal{H}_S}
\right].
\]
Combining this with
\cref{eq:steering-statistic-identity} gives
\[
\frac{\mathrm{d}}{\mathrm{d}\alpha}
J_{\psi,x}(\alpha)
=
\operatorname{Var}_{\mathbf{u}\sim
\pi_{\lambda_\alpha}(\cdot\mid x)}
\left[
\psi\bigl(F_{\mathbf{u}}(x)\bigr)
\right].
\]
Since variance is nonnegative,
\[
\frac{\mathrm{d}}{\mathrm{d}\alpha}
J_{\psi,x}(\alpha)
\ge 0.
\]

In particular, for any
\(\alpha_1\le \alpha_2\) such that
\(\lambda_\alpha\in\Theta\) for all
\(\alpha\in[\alpha_1,\alpha_2]\),
\[
\begin{aligned}
J_{\psi,x}(\alpha_2)
-
J_{\psi,x}(\alpha_1)
&=
\int_{\alpha_1}^{\alpha_2}
\frac{\mathrm{d}}{\mathrm{d}\alpha}
J_{\psi,x}(\alpha)\,\mathrm{d}\alpha
\\
&=
\int_{\alpha_1}^{\alpha_2}
\operatorname{Var}_{\mathbf{u}\sim
\pi_{\lambda_\alpha}(\cdot\mid x)}
\left[
\psi\bigl(F_{\mathbf{u}}(x)\bigr)
\right]
\,\mathrm{d}\alpha
\\
&\ge 0.
\end{aligned}
\]
Therefore \(J_{\psi,x}(\alpha)\) is nondecreasing along the steering path
\(\lambda_\alpha=\lambda(x)+\alpha\lambda_\psi(x)\).
\end{proof}

\newpage

\section{Detailed Derivations}
\subsection{KL Projection and the Induced Steering Path}
\label[appendixsubsection]{app:kl-projection}

We provide a derivation of the reverse-KL projection used in~\cref{eq:linear-path}.
Consider the exponential family induced by the base measure $\mu$
\begin{equation*}
    \frac{\mathrm d\pi_\lambda}{\mathrm d\mu}(\mathbf u)
    =
    \exp\!\left(
        \langle \lambda, S(\mathbf u) \rangle_{\mathcal H_S} - A(\lambda)
    \right),
    \qquad
    \lambda \in \Theta \subseteq \mathcal H_S,
\end{equation*}
where $A(\lambda)=\log\mathbb E_{\mathbf  u\sim\pi_0}\left[e^{\langle \lambda,S(\mathbf  u)\rangle_{\mathcal H_S}}\right]$
is the log-partition function. We assume that the exponential family is regular and minimal, so that $A$ is strictly convex on $\Theta$. Its expectation coordinate is
\begin{equation*}
    \varphi(\lambda)=
    \nabla A(\lambda)
    =
    \mathbb E_{\mathbf u\sim\pi_\lambda}[S(\mathbf u)],
    \qquad
    \varphi = \varphi(\Theta).
\end{equation*}

For a fixed state $x\in\mathcal X$, recall the state-conditioned steering direction
\begin{equation*}
    \lambda_\psi(x)
    =
    L_\psi^* \lambda(x)
    \in \mathcal H_S.
\end{equation*}
By construction, the expected propagated QoI is affine in the dual coordinate:
\begin{align*}
    \mathbb E_{\mathbf u\sim\pi_\lambda}
    \bigl[\psi(F_{\mathbf u}(x))\bigr]
    &=
    \left\langle
        \lambda(x),
        L_\psi
        \mathbb E_{\mathbf u\sim\pi_\lambda}[S(\mathbf u)]
    \right\rangle_{\mathcal H_S}
    \\
    &=
    \left\langle
        L_\psi^*\lambda(x),
        \varphi(\lambda)
    \right\rangle_{\mathcal H_S}
    \\
    &=
    \left\langle
        \lambda_\psi(x),
        \varphi(\lambda)
    \right\rangle_{\mathcal H_S}.
\end{align*}
Hence, a target expected QoI level $c\in\mathbb R$ defines the affine hyperplane
\begin{equation*}
    \varphi_{\psi,x}(c)
    =
    \left\{
        \varphi\in\varphi:
        \langle \lambda_\psi(x),\varphi\rangle_{\mathcal H_S}=c
    \right\}.
\end{equation*}

\paragraph{KL divergence as a dual Bregman divergence.}
For any $\lambda,\lambda_0\in\Theta$, the KL divergence between two members of the same exponential family satisfies
\begin{align*}
    D_{\mathrm{KL}}(\pi_\lambda\|\pi_{\lambda_0})
    &=
    \mathbb E_{\pi_\lambda}
    \left[
        \log
        \frac{\mathrm{d}\pi_\lambda}{\mathrm{d}\pi_{\lambda_0}}(\mathbf u)
    \right]
    \\
    &=
    \left\langle
        \lambda-\lambda_0,
        \varphi(\lambda)
    \right\rangle_{\mathcal H_S}
    -
    A(\lambda)
    +
    A(\lambda_0)
    \\
    &=    D_A(\lambda_0,\lambda),
\end{align*}
where $D_A(\lambda_0,\lambda):=A(\lambda_0)-A(\lambda)-\left\langle\nabla A(\lambda), \lambda_0-\lambda \right\rangle$ is the Bregman divergence generated by $A$.

Let $A^*$ denote the convex conjugate of $A$,
\begin{equation*}
    A^*(\varphi)
    :=
    \sup_{\lambda\in\Theta}
    \left\{
        \langle \lambda,\varphi\rangle_{\mathcal H_S}-A(\lambda)
    \right\}.
\end{equation*}
The Legendre correspondence gives
\begin{equation*}
    \varphi=\nabla A(\lambda),
    \qquad
    \lambda=\nabla A^*(\varphi),
\end{equation*}
and therefore
\begin{equation}
    D_{\mathrm{KL}}(\pi_\lambda\|\pi_{\lambda_0})
    =
    D_{A^*}
    \bigl(
        \varphi(\lambda),
        \varphi(\lambda_0)
    \bigr).
    \label{eq:kl-dual-bregman}
\end{equation}

\paragraph{Projection onto a constant-QoI hyperplane.}
We now set $\lambda_0=\lambda(x)$ and consider
\begin{equation*}
    \widehat\lambda
    \in
    \arg\min_{\lambda:\,
        \varphi(\lambda)\in\varphi_{\psi,x}(c)}
    D_{\mathrm{KL}}
    \bigl(
        \pi_\lambda\|
        \pi_{\lambda(x)}
    \bigr).
\end{equation*}
Using~\cref{eq:kl-dual-bregman}, this is equivalently the Bregman projection
\begin{equation*}
    \widehat\varphi
    \in
    \arg\min_{\varphi\in\varphi}
    D_{A^*}
    \bigl(
        \varphi,\varphi(\lambda(x))
    \bigr)
    \quad
    \text{subject to}
    \quad
    \langle\lambda_\psi(x),\varphi\rangle_{\mathcal H_S}=c.
\end{equation*}

Introduce a Lagrange multiplier $\beta\in\mathbb R$:
\begin{equation*}
    \mathcal L(\varphi,\beta)
    =
    D_{A^*}
    \bigl(
        \varphi,\varphi(\lambda(x))
    \bigr)
    +
    \beta
    \left(
        \langle\lambda_\psi(x),\varphi\rangle_{\mathcal H_S}-c
    \right).
\end{equation*}
The first-order optimality condition with respect to $\varphi$ is
\begin{equation*}
    \nabla A^*(\widehat\varphi)
    -
    \nabla A^*
    \bigl(
        \varphi(\lambda(x))
    \bigr)
    +
    \beta\lambda_\psi(x)
    =
    0.
\end{equation*}
By Legendre duality,
\begin{equation*}
    \nabla A^*(\widehat\varphi)
    =
    \widehat\lambda,
    \qquad
    \nabla A^*
    \bigl(
        \varphi(\lambda(x))
    \bigr)
    =
    \lambda(x),
\end{equation*}
and hence $\widehat\lambda=\lambda(x)-\beta\lambda_\psi(x)$. Finally, writing $\alpha=-\beta$ gives
\begin{equation*}
    \widehat\lambda
    =
    \lambda(x)+\alpha\lambda_\psi(x).
\end{equation*}


\subsection{Oracle Representation and Measurement Map}
\label[appendixsubsection]{app:oracle-derivation}

We first write the dynamics in control-affine form. Since the system is driftless,
\[
\dot{x}
=
f_1(x)v + f_2(x)\omega,
\qquad
f_1(x)
=
\begin{bmatrix}
1\\
\theta\\
0
\end{bmatrix},
\qquad
f_2(x)
=
\begin{bmatrix}
0\\
0\\
1
\end{bmatrix}.
\]

\paragraph{Observable and exact oracle representation.}
We consider the QoI $\psi(x)=p_y$. The first-order Lie derivatives are
\[
\mathcal L_{f_1}\psi(x)=\nabla\psi(x)^\top f_1(x)=\theta,
\qquad
\mathcal L_{f_2}\psi(x)=\nabla\psi(x)^\top f_2(x)=0.
\]
At second order,
\[
\mathcal L_{f_2}\mathcal L_{f_1}\psi(x)
=\mathcal L_{f_2}\theta=1,
\]
whereas
\[
\mathcal L_{f_1}\mathcal L_{f_1}\psi
=
\mathcal L_{f_1}\mathcal L_{f_2}\psi
=
\mathcal L_{f_2}\mathcal L_{f_2}\psi
=0.
\]
Since the only second-order iterated Lie derivative that is not identically zero is constant, all iterated Lie derivatives of order three or higher vanish identically.

Strictly speaking, under the augmented formulation with $u^0 \equiv 1$, the input should be written as $\mathbf{u}=(1,v,\omega)$. However, since the time component plays no role in the control parametrization, we suppress it for notational simplicity and write $\mathcal A=\{1, 2\}$ as the control alphabet, with $1$ and $2$ corresponding to $v$ and $\omega$, respectively, and let $\mathcal A^\star$ denote the set of all finite words over $\mathcal A$, including the empty word $\emptyset$. We use the convention
\[
\mathcal L_{(i_1,\ldots,i_k)}\psi
:=
\mathcal L_{f_{i_1}}\cdots\mathcal L_{f_{i_k}}\psi,
\qquad
\mathcal L_\emptyset\psi:=\psi,
\]
Let $S(\mathbf u)=\bigl(S(\mathbf{u})^I_{0,H}\bigr)_{I\in\mathcal A^\star}$ denote the full signature of the control, with $S^\emptyset_{0,H}(\mathbf u)=1$. The only words whose Chen--Fliess coefficient functions are not identically zero are
\[
\mathcal W_\psi=\{\emptyset,1,21\}.
\]
In particular,
\[
S(\mathbf u)^1_{0,H}=\int_0^H v(t)\,\mathrm{d}t,
\qquad
S(\mathbf u)^{21}_{0,H}
=
\int_0^H\int_0^{t_2}\omega(t_1)v(t_2)\,\mathrm{d}t_1\,\mathrm{d}t_2.
\]
Hence the Chen--Fliess expansion terminates exactly:
\[
\psi(F_{\mathbf u}(x))
=
\sum_{I\in\mathcal A^\star}
\mathcal L_I\psi(x)S(\mathbf u)^I_{0,H}
=
p_y+\theta S(\mathbf u)^1_{0,H}+S^{21}_{0,H}(\mathbf u).
\]
Define
\[
\lambda_\psi(x)
=
\bigl(\lambda_\psi^I(x)\bigr)_{I\in\mathcal A^\star},
\qquad
\lambda_\psi^I(x):=\mathcal L_I\psi(x).
\]
Since $\lambda_\psi(x)$ has finite support $\mathcal W_\psi$, its pairing with the full signature is a finite sum:
\[
\psi(F_{\mathbf u}(x))
=
\langle \lambda_\psi(x),S(\mathbf u)\rangle_{\mathcal H_S}.
\]
The only coordinates of $\lambda_\psi(x)$ that are not identically zero are
\[
\lambda_\psi^\emptyset(x)=p_y,
\qquad
\lambda_\psi^1(x)=\theta,
\qquad
\lambda_\psi^{21}(x)=1.
\]

Let $\Pi_{\mathcal W_\psi}$ denote projection onto these three active coordinates. Then
\[
\Pi_{\mathcal W_\psi}S(\mathbf u)
=
\begin{bmatrix}
1\\
S(\mathbf u)^1_{0,H}\\
S(\mathbf u)^{21}_{0,H}
\end{bmatrix},
\qquad
\Pi_{\mathcal W_\psi}\lambda_\psi(x)
=
\begin{bmatrix}
p_y\\
\theta\\
1
\end{bmatrix}.
\]
With a slight abuse of notation, whenever only these active coordinates are considered, we write
\[
S(\mathbf u)
=
\begin{bmatrix}
1\\
S(\mathbf u)^1_{0,H}\\
S(\mathbf u)^{21}_{0,H}
\end{bmatrix},
\qquad
\lambda_\psi(x)
=
\begin{bmatrix}
p_y\\
\theta\\
1
\end{bmatrix}.
\]

\paragraph{Signature measurement map $L_{\psi}$.}
For a learned VLA representation \(\lambda\), the particular realization of \(L_{\psi}\) may depend on how the VLA model is trained, and more specifically, on how much physical information is encoded in \(\lambda\).  In particular, our oracle experiment does not assume or learn a specific realization of \(L_{\psi}\); rather, \(L_{\psi}\) is treated as given, and the experiment is intended to demonstrate the existence of such a linear correspondence. More generally, how the representation dimensionality scales with the number of QoIs, as well as possible intersections among the associated QoI subspaces, requires further consideration.


\paragraph{Numerical setting.}
We denote by $\mathbf{u}$ an entire action chunk over the horizon $[0,H]$. Each chunk is piecewise
constant over $K=8$ equal-length intervals of duration $\Delta t = H/K$ and is parameterized by a forward
speed $s$ and a turning rate $r$,
\[
    v(t)=s,
    \qquad
    \omega(t)=
    \begin{cases}
        r, & t < H/2,\\
        -r, & t \geq H/2,
    \end{cases}
\]
so that each chunk turns and then counter-turns, producing lateral motion while approximately restoring
the terminal heading. Sweeping $s$ over $11$ uniform points in $[0.75,1.25]$ and $r$ over $61$ uniform
points in $[-2.2,2.2]$ yields $671$ candidates. For visual reference, we place a circular obstacle
centered at $p_{\mathrm{obs}} = (0.5, 0)$ with radius $r_{\mathrm{obs}} = 0.08$; it is not included in
$C_{\mathrm{task}}$. The change of coordinates and the goal-reaching cost are
\[
    M=
    \begin{bmatrix}
        0.10 & -0.04 & 0.18\\
        0.14 &  0.03 & 0\\
        0.02 &  0.16 & 0
    \end{bmatrix},
    \qquad
    \lambda(x_0)=M\lambda_\psi(x_0)=(0.18,0,0)^\top,
\]
Since \(M^\top M^{-\top}=I\), the bilinear relation $\psi(F_\mathbf{u}(x)) = \langle \lambda(x),\gamma_\psi(\mathbf{u})\rangle_{\mathcal H_S}$
holds exactly. Since $\lambda_\psi(x_0)=(0,0,1)^\top$, the parameter $\lambda(x_0)$ contributes only the
constant $\langle\lambda(x_0),S(\mathbf{u})\rangle_{\mathcal H_S}\equiv0.18$ to the policy log-density;
hence, after normalization, the unsteered policy $\pi_{\lambda(x_0)}$ coincides with the base measure $\mu$.
Along the steering path, $L_\psi^\top\lambda(x)=M^{-1}M\lambda_\psi(x)=\lambda_\psi(x)$, so
$\lambda_\psi(x_0)=(0,0,1)^\top$ and
\[
    \langle\lambda_\psi(x_0),S(\mathbf{u})\rangle_{\mathcal H_S}
    = S^{21}_{0,H}(\mathbf{u})
    = \psi(F_\mathbf{u}(x_0)),
\]
the last equality because $p_y=\theta=0$ at $x_0$. The log-density being affine in $\lambda$, the shift
by $\alpha\lambda_\psi(x_0)$ factors out,
\[
    \pi_{\lambda_\alpha}(\mathbf{u}\mid x_0)
    \propto \mu(\mathbf{u})\,
    e^{\langle\lambda(x_0),S(\mathbf{u})\rangle_{\mathcal H_S}}\,
    e^{\alpha\langle\lambda_\psi(x_0),S(\mathbf{u})\rangle_{\mathcal H_S}}
    \propto \pi_{\lambda(x_0)}(\mathbf{u}\mid x_0)\,
    \exp\bigl(\alpha\,\psi(F_{\mathbf{u}}(x_0))\bigr).
\]

\newpage

\newpage


\section{Experiments on a Pretrained VLA}
\label{sec:oft_experiments}
\subsection{Setup}
\label{sec:oft_setup}

\paragraph{Model and representation.}
OpenVLA-OFT decodes an action chunk in parallel: the last-layer hidden states of the Llama-2 backbone at the positions of the action tokens are mapped by a small MLP (the action head) to $H_{\mathrm{chunk}}=8$ continuous 7-dimensional actions. We take these hidden states as the representation, $\lambda(x)\in\mathbb R^{56\times4096}$ (8 steps $\times$ 7 action dimensions), and work with the
latent
\begin{equation}
  z(x) := \big(1,\ P\lambda(x)\big)\in \mathbb R^q,
  \label{eq:oft_lambda}
\end{equation}
where $P$ keeps the leading principal components of $\lambda(x)$, fitted on training states. The representation $\lambda(x)$ of the theory corresponds to the full representation; $z(x)$ is its truncation to the leading principal coordinates, and we apply \Cref{def:linear_probing} with $z$ in place of $\lambda$. The policy chunk is a deterministic function of the representation, $\mathbf u(x) = \mathrm{head}(\lambda(x))$: re-applying the head to the stored $\lambda(x)$ reproduces the executed chunk exactly. Representation interventions therefore act directly on the action, without passing through the language model again.

\begin{figure}[t]
    \centering
    \includegraphics[width=0.9\linewidth]{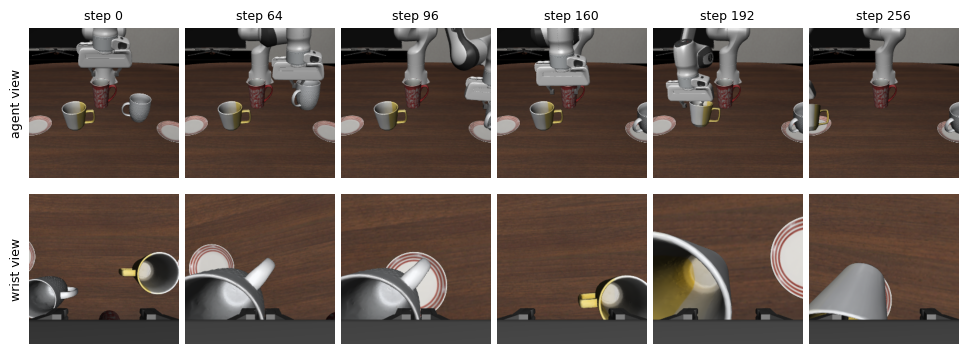}
    \caption{\textbf{LIBERO-10 task}: \emph{``put the white mug on the left plate and put the yellow-and-white mug on the right plate''}. Frames of a successful closed-loop OpenVLA-OFT rollout (top: agent view, bottom: wrist view), with the control step above each column.}
    \label{fig:Libero_figure}
\end{figure}

\paragraph{Task and data.}
We use the checkpoint fine-tuned on LIBERO-10 and the task shown in ~\Cref{fig:Libero_figure}. We roll out the policy from all 50 initial layouts (46 successful) and take a snapshot every 16 control steps (788 states). At every second snapshot (404 states) we restore the simulator and execute 32 candidate chunks of $H_{\mathrm{chunk}}=8$ steps: the policy chunk and 31 smooth perturbations of it.

\paragraph{Quantities of interest.}
The QoIs are the 3D positions of the two mugs ($\mathrm{mug}_1$: white, $\mathrm{mug}_2$: yellow-and-white) and of the end effector (EE); the EE coordinates serve as a control, since under delta end-effector actions their change follows directly from the action. For steering, we additionally use four task-level QoIs: the negative horizontal distance $-d(\mathrm{mug}_j, \mathrm{plate})$ of each
mug to its target plate and the negative distance $-d(\mathrm{EE}, \mathrm{mug}_j)$ between the EE and each mug, so that increasing the QoI means moving towards the goal; the target plate of each mug is the one it ends on in all successful episodes. All results are on held-out episodes.

\subsection{Action-Conditioned Linear Probing}
\label{sec:oft_probing}

We split the measurement map along the signature coordinates, $L_\psi=[L_\psi^\emptyset,\ L_\psi^{1:2}]$,
into the static probe and the action-conditioned increment,
\begin{align}
  \psi(x) &\approx \langle z(x),\, L_\psi^\emptyset\rangle_{\mathbb R^q}, \label{eq:oft_static}\\
  \Delta\psi(x;\mathbf u) := \psi(F_{\mathbf u}(x)) - \psi(x) &\approx \big\langle z(x),\, L_\psi^{1:2}\,S_{1:2}(\mathbf u)\big\rangle_{\mathbb R^q},
  \label{eq:oft_increment}
\end{align}
whose sum is the probe of \Cref{def:linear_probing} truncated at level two,
\begin{equation}
  \hat\psi\big(F_{\mathbf u}(x)\big)
  :=\big\langle z(x),\, L_\psi\, S_{\le 2}(\mathbf u)\big\rangle_{\mathbb R^q}
  = \langle z(x),\, L^\emptyset_\psi\rangle_{\mathbb R^q}
  + \big\langle z(x),\, L^{1:2}_\psi\, S_{1:2}(\mathbf u)\big\rangle_{\mathbb R^q} .
  \label{eq:smolvla_probe}
\end{equation}
The static term uses 256 principal components and the increment the leading 32. For the increment we compare the bilinear probe with a \emph{linear} probe $\langle z(x),a\rangle_{\mathbb R^q}$ that ignores the chunk, an \emph{additive} probe $\langle z(x),a\rangle_{\mathbb R^q} + \langle b, S_{1:2}(\mathbf{u}) \rangle_{\mathbb R^q}$, an \emph{MLP readout} $\langle \mathrm{MLP}(z(x)), S_{1:2}(\mathbf u)\rangle_{\mathbb R^q}$ of the signature coefficients, and an \emph{unstructured} $\mathrm{MLP}(z(x),\mathbf u)$ on the representation and the raw chunk. \Cref{tab:oft_propagated} reports the propagated QoI, which is the quantity of \Cref{def:linear_probing} and \Cref{tab:oft_increment} the increment alone.

\begin{table}[t]
\centering
\caption{\textbf{Propagated QoI from the representation alone (OpenVLA-OFT).} $R^2$ of $\psi(F_{\mathbf u}(x))$ on held-out episodes. \emph{Static only (s.o.)} uses the static probe $\langle z(x), L_\psi^\emptyset\rangle$ as the prediction, i.e.\ it assumes that the chunk does not change the QoI; the remaining columns add one of the increment probes of~\Cref{tab:oft_increment}, and the bilinear column is the probe of \Cref{def:linear_probing}. The first column is the $R^2$ of the static probe for the current value $\psi(x)$. EE coordinates are control QoIs.}
\label{tab:oft_propagated}
\setlength{\tabcolsep}{3.2pt}
\begin{tabular}{lccccccc}
\hline
 & static & \multicolumn{6}{c}{propagated $\psi(F_{\mathbf u}(x))$} \\
QoI $\psi$ & $\psi(x)$ & s.o. & $+$ linear & $+$ additive & $+$ \textbf{bilinear}
 & $+$ MLP readout & $+$ unstructured \\
\hline
EE $p_x$ (control) & .836 & .510 & .770 & .820 & .822 & .820 & .820 \\
EE $p_y$ (control) & .833 & .809 & .848 & .851 & .851 & .851 & .851 \\
EE $p_z$ (control) & .944 & .593 & .902 & .943 & .943 & .946 & .945 \\
\hline
mug$_1$ $p_x$ & .777 & .671 & .753 & .763 & .767 & .768 & .766 \\
mug$_1$ $p_y$ & .865 & .821 & .853 & .858 & .861 & .861 & .862 \\
mug$_1$ $p_z$ & .956 & .636 & .918 & .938 & .946 & .949 & .949 \\
mug$_2$ $p_x$ & .580 & .453 & .565 & .586 & .593 & .593 & .600 \\
mug$_2$ $p_y$ & .814 & .752 & .842 & .844 & .846 & .847 & .846 \\
mug$_2$ $p_z$ & .965 & .620 & .921 & .940 & .959 & .959 & .960 \\
\hline
\end{tabular}
\end{table}

\begin{table}[t]
\centering
\caption{\textbf{Action-conditioned increment (OpenVLA-OFT).} $R^2$ of $\Delta\psi(x;\mathbf u)$ on
held-out episodes, with $S=S_{1:2}(\mathbf u)$ and $H_{\mathrm{chunk}}=8$.}
\label{tab:oft_increment}
\setlength{\tabcolsep}{4pt}
\begin{tabular}{lccccc}
\hline
 & linear & additive & \textbf{bilinear} & MLP readout & unstructured \\
QoI $\psi$
 & $\langle z, a\rangle$
 & $\langle z, a\rangle + \langle b, S\rangle$
 & $\langle z, L^{1:2}_\psi S\rangle$
 & $\langle \mathrm{MLP}(z), S\rangle$
 & $\mathrm{MLP}(z,\mathbf u)$ \\
\hline
EE $p_x$ (control) & .829 & .983 & .981 & .982 & .981 \\
EE $p_y$ (control) & .912 & .995 & .995 & .993 & .995 \\
EE $p_z$ (control) & .867 & .986 & .986 & .987 & .988 \\
\hline
mug$_1$ $p_x$ & .878 & .933 & .960 & .970 & .973 \\
mug$_1$ $p_y$ & .806 & .895 & .959 & .962 & .967 \\
mug$_1$ $p_z$ & .825 & .877 & .918 & .933 & .941 \\
mug$_2$ $p_x$ & .656 & .801 & .901 & .916 & .934 \\
mug$_2$ $p_y$ & .908 & .953 & .979 & .983 & .983 \\
mug$_2$ $p_z$ & .813 & .864 & .923 & .940 & .939 \\
\hline
\end{tabular}
\end{table}

\paragraph{The propagated QoI is predicted from the representation alone.}
The probe of \Cref{def:linear_probing} predicts the mug positions after the chunk with $R^2$ between $0.59$ and $0.96$ (\Cref{tab:oft_propagated}). The increment is essential for the mug heights: the current value alone gives $0.64$ and $0.62$, the bilinear probe $0.95$ and $0.96$. The two flexible alternatives change $R^2$ by at most $0.01$. The remaining error is dominated by the static part, which reaches $0.58$--$0.97$ for the mug coordinates.

\paragraph{The action-conditioned increment is linearly encoded.}
The bilinear probe predicts how a candidate chunk moves the mugs with $R^2$ between $0.90$ and $0.98$ (\Cref{tab:oft_increment}), against $0.80$--$0.95$ for the additive probe; the gain is largest for mug$_2$ $p_x$ ($+0.10$) and for mug$_1$ $p_y$ and mug$_2$ $p_z$ ($+0.06$). For the EE coordinates the additive and bilinear probes coincide, as expected for a control QoI. Replacing the linear readout by an MLP, or the bilinear structure by an unstructured MLP, improves $R^2$ by at most about $0.03$. Because the policy chunk is itself a function of $z$, the linear probe
already captures the effect of the policy's own chunk; the signature term is needed for the effect of the deviations from it, and the bilinear term for how that effect depends on the state.

\subsection{Steering Along the Probe Direction}
\label{sec:oft_steering}
\Cref{def:linear_steering} asks for a direction in representation space along which the QoI after the
chunk increases monotonically. We ask whether the probe of~\Cref{sec:oft_probing} supplies such a
direction for OpenVLA-OFT, and which information besides the probe is needed to use it.

\paragraph{Intervention.}
We keep the physical state $x$ fixed and move only the representation that the action head receives, along a fixed direction $z_\psi(x)$:
\begin{equation}
  z_\alpha = z(x) + \alpha\, z_\psi(x), \quad
  \mathbf u(z_\alpha) := \mathrm{head}\big(\lambda(x) + \alpha P^\top\mathrm{diag}(\sigma)\, z_\psi(x)\big).
  \label{eq:oft_steer}
\end{equation}
Let $J_{\psi, x}(\alpha) = \psi\big(F_{\mathbf u(z_\alpha)}(x)\big)$. The direction lives in the leading 32 principal components and is mapped back to $\lambda\in\mathbb R^{56\times4096}$ through the principal directions, with $\sigma$ the standard
deviations of the principal coordinates; all other components of $\lambda$ are unchanged, and $\mathbf u(z)$ denotes the resulting chunk. Because the chunk is exactly $\mathrm{head}(\lambda)$, the steered chunk is obtained by re-applying the head to the displaced representation, without another pass through the language model. Because the policy is deterministic, $J_{\psi, x}(\alpha)$ is determined by a single
simulated chunk for each $\alpha$; there is no sampling noise.

In~\Cref{eq:linear-path} the direction is $\lambda_\psi(x)$ itself, which presupposes that the representation and the probe direction live in the same space, as under the identification of \Cref{def:signatureGLM}. Here, \(z(x)\) has one entry per principal coordinate, whereas \(\lambda_\psi(x)\) has one entry per signature coordinate; hence, \(\lambda_\psi(x)\) cannot be added directly to \(z(x)\). Instead, \(\lambda_\psi(x)\) is evaluated at the original latent representation \(z(x)\) and then held fixed. We therefore parameterize the steering path by a latent-space direction \(z_\psi(x)\), as in~\Cref{eq:oft_steer}, and ask which \(z_\psi(x)\) realizes the desired signature-space direction \(\lambda_\psi(x)\).

\paragraph{From the probe to a representation direction.}
The probe describes the second half of the chain
\begin{equation}
  z \ \xrightarrow{\ \mathrm{head}\ }\ \mathbf u(z) \ \xrightarrow{\ \text{physics}\ }\
  \psi\big(F_{\mathbf u(z)}(x)\big),
  \qquad
  \hat\psi\big(F_{\mathbf u}(x)\big) = \big\langle \lambda_\psi(x),\, S_{\le2}(\mathbf u)\big\rangle_{\mathbb R^q},
  \label{eq:oft_chain_def}
\end{equation}
where $\lambda_\psi(x) := L_\psi^\top z(x)$ has one entry per signature coordinate, i.e.\ it describes how the QoI depends on the signature of the chunk. The intervention, however, acts on the first half: moving $z$ changes the chunk through the head. Writing
\begin{equation}
  \phi(z) := S_{\le2}\big(\mathbf u(z)\big), \qquad
  D\phi(z) := \frac{\partial\, S_{\le2}(\mathbf u(z))}{\partial z},
  \label{eq:oft_phi}
\end{equation}
for the signature of the chunk produced from $z$ and its Jacobian, the probe's prediction along
the intervention is $\hat J_{\psi, x}(\alpha) = \langle \lambda_\psi(x),\, \phi(z_\alpha)\rangle$ and, by
the chain rule,
\begin{equation}
  \begin{aligned}
    \frac{\mathrm d\hat J_{\psi, x}}{\mathrm d\alpha}\Big|_{\alpha=0}
    &= \Big\langle \lambda_\psi(x),\, \frac{\mathrm d}{\mathrm d\alpha} S_{\le2}\big(\mathbf u(z_\alpha)\big)\Big|_{\alpha=0} \Big\rangle_{\mathbb R^q} \\
     &= \big\langle \lambda_\psi(x),\, D\phi\big(z(x)\big)\, z_\psi(x)\big\rangle_{\mathbb R^q} \\
    &= \big\langle D\phi\big(z(x)\big)^{\!\top} \lambda_\psi(x),\, z_\psi(x)\big\rangle_{\mathbb R^q} .
  \end{aligned}
  \label{eq:oft_chain}
\end{equation}
The intervention changes the action, not the physical state, and it is the physical state that determines how a chunk moves the objects. Only the non-empty words of $S$ contribute, since $S^\emptyset=1$ and the pure-time words are constant for a fixed horizon; the corresponding rows of $D\phi$ vanish. The static part $L_\psi^\emptyset$ of the probe therefore plays no role in~\Cref{eq:oft_chain}.

\paragraph{Relation to~\Cref{sec:steering}.}
The map $\phi$ in~\Cref{eq:oft_phi} is the deterministic counterpart of the mean coordinates $\varphi(\lambda)=\mathbb E_{\pi_\lambda}[S(\mathbf u)]$ of~\Cref{sec:steering}. Under the identification of~\Cref{def:signatureGLM}, the representation is the natural parameter, $D\phi = \nabla^2 A$ is symmetric positive semidefinite, and \Cref{eq:oft_chain} with $z_\psi=\lambda_\psi$ gives $\mathrm d\hat J_{\psi, x}/\mathrm d\alpha = \lambda_\psi^\top\nabla^2 A\,\lambda_\psi = \mathrm{Var}[\langle\lambda_\psi,S\rangle]\ge0$ for every $\alpha$: the probe direction itself is an ascent direction and $D\phi$ need not be known (\Cref{thm:linear_steering}). For a learned action head, $D\phi$ is a general $|S|\times\dim z$ matrix,
neither symmetric nor positive semidefinite, and $\lambda_\psi$ does not even live in the same space as $z$; the probe direction has to be transported into latent space through $D\phi$.

\paragraph{Steering directions.}
We compare five directions, all fitted or computed without the test episodes (except the oracle):
\begin{itemize}
  \item $D\phi^\top\lambda_\psi$: the probe direction transported through the head,~\Cref{eq:oft_chain}. It is the gradient of the probe's prediction $\hat J_{\psi, x}$ with respect to the representation and uses only the probe and the model's own action head; no simulator query is needed.
  \item $\gamma_\psi = \nabla_z\,\hat\psi(F_{\mathbf u_0}(x))$ with the chunk held fixed at the policy chunk $\mathbf u_0$ (increment part only), i.e.\ $\gamma_\psi(\mathbf{u}) = (\partial/\partial z)
        \langle z, L_\psi^{1:2}S_{1:2}(\mathbf u)\rangle = L_\psi^{1:2}S_{1:2}(\mathbf u)$. This is the direction obtained by differentiating the probe alone, as in steering with a linear probe in language models. It changes what the probe reads from the representation for the same chunk, but ignores how the head responds to the change.
  \item static: the static probe direction $L_\psi^\emptyset$, which changes what the representation encodes about the current value $\psi(x)$.
  \item simulator gradient: $\nabla_z J_{\psi, x}$ at $\alpha=0$, obtained by central differences of the simulated QoI along the 32 principal directions. It uses no probe and serves as an oracle for the best local linear steering direction in this subspace.
  \item random: three random directions.
\end{itemize}

\paragraph{What the construction guarantees.}
Along $z_\psi = D\phi^\top\lambda_\psi$, the probe's prediction increases to first order, $\hat J_{\psi, x}(\alpha) = \hat J_{\psi, x}(0) + \alpha\,\|D\phi^\top\lambda_\psi\|^2 + O(\alpha^2)$, for any differentiable head. Whether the true QoI increases depends on the probe residual $r(\mathbf u) = \psi(F_{\mathbf u}(x)) - \hat\psi(F_{\mathbf u}(x))$: the true slope is $\langle D\phi^\top\lambda_\psi + \nabla_z r(\mathbf u(z)), \, z_\psi \rangle$, so the probe must be accurate for the chunk variations that the intervention induces, which is not implied by a high $R^2$ on the data. Monotonicity beyond first order is not guaranteed and is tested empirically.

\paragraph{Computation.}
$D\phi$ is obtained by central differences along the 32 principal directions (step $0.25$ standard deviations), which requires only re-applying the action head to the displaced representations. Re-applying the head to the stored (float16) representation reproduces the executed policy chunk up to $3.3\cdot10^{-3}$ on the arm dimensions, far below the size of the perturbations. The simulator gradient uses the same displacements, executed in the simulator. All directions are rescaled to the same displacement of $\lambda$ per unit $\alpha$, $0.5\%$ of $\|\lambda(x)\|$, and $\alpha\in\{-8,-4,-2,0,2,4,8\}$, i.e.\ up to $4\%$ of the representation norm.

\paragraph{Evaluation protocol.}
Episodes are split into two halves. The principal components, both probes and all directions are fitted on one half, and steering is evaluated at states of the other half; the roles are then swapped. For each QoI we steer at the 10 states per half with the largest spread of the QoI among the candidate chunks; for the mug heights and the mug--plate distances we require that at least half of the candidate chunks change the QoI by more than 2\,mm, which selects states in which the mug responds
to the action at all. We report (i)~the share of states with a monotone response (Spearman correlation between $\alpha$ and $J_{\psi, x}(\alpha)$ at least $0.8$), (ii)~the median effect $J_{\psi, x}(8)-J_{\psi, x}(-8)$, (iii)~the share of states in which a direction's effect exceeds that of all three random directions, (iv)~the median cosine with the
simulator gradient, and (v)~for $D\phi^\top\lambda_\psi$, the ratio of the measured slope $(J_{\psi, x}(2)-J_{\psi, x}(-2))/4$ to the probe-predicted slope $\langle\lambda_\psi, D\phi\,z_\psi\rangle$. \Cref{tab:oft_steering} lists (i)--(iv) for every QoI, and \Cref{fig:oft_curves} shows the full response curves.

\begin{table}[t]
\centering
\caption{\textbf{Steering OpenVLA-OFT along probe directions.} Share of held-out states with a monotone response of $J_{\psi, x}(\alpha)$ (Spearman $\rho\ge0.8$), median effect $J_{\psi, x}(8)-J_{\psi, x}(-8)$ (mm), share of states in which the direction beats all three random directions of the same displacement, and median cosine with the simulator gradient; 20 states per QoI. The last four QoIs are negative distances (each mug to its target plate horizontally, the EE to each mug), so a positive effect means moving towards the goal.}
\label{tab:oft_steering}
\setlength{\tabcolsep}{4pt}
\begin{tabular}{llcccc}
\hline
QoI & direction & monotone & effect (mm) & beats random & cos to sim.\ grad. \\
\hline
EE $p_z$ (control) & simulator gradient        & 100\% & $+9.5$  & 100\% & $1.00$ \\
                 & $D\phi^\top\lambda_\psi$  & 100\% & $+9.8$  & 100\% & $1.00$ \\
                 & $\gamma_\psi$             & 50\%  & $+0.4$  & 15\%  & $0.03$ \\
                 & static                    & 0\%   & $-4.2$  & 0\%   & $-0.33$ \\
                 & random                    & 52\%  & $+0.6$  & --    & -- \\
\hline
mug$_1$ $p_z$      & simulator gradient        & 100\% & $+12.9$ & 100\% & $1.00$ \\
                 & $D\phi^\top\lambda_\psi$  & 100\% & $+12.9$ & 100\% & $1.00$ \\
                 & $\gamma_\psi$             & 65\%  & $+2.1$  & 30\%  & $0.13$ \\
                 & static                    & 0\%   & $-2.2$  & 5\%   & $-0.14$ \\
                 & random                    & 45\%  & $-0.6$  & --    & -- \\
\hline
mug$_2$ $p_z$      & simulator gradient        & 100\% & $+11.6$ & 100\% & $1.00$ \\
                 & $D\phi^\top\lambda_\psi$  & 100\% & $+11.2$ & 100\% & $1.00$ \\
                 & $\gamma_\psi$             & 100\% & $+5.5$  & 80\%  & $0.27$ \\
                 & static                    & 0\%   & $-2.1$  & 5\%   & $-0.12$ \\
                 & random                    & 40\%  & $-1.0$  & --    & -- \\
\hline
$-d$(mug$_1$, plate) & simulator gradient    & 100\% & $+12.3$ & 100\% & $1.00$ \\
                 & $D\phi^\top\lambda_\psi$  & 100\% & $+11.9$ & 100\% & $0.98$ \\
                 & $\gamma_\psi$             & 100\% & $+7.6$  & 70\%  & $0.37$ \\
                 & static                    & 0\%   & $-11.3$ & 0\%   & $-0.61$ \\
                 & random                    & 42\%  & $-1.8$  & --    & -- \\
\hline
$-d$(mug$_2$, plate) & simulator gradient    & 100\% & $+12.4$ & 100\% & $1.00$ \\
                 & $D\phi^\top\lambda_\psi$  & 100\% & $+12.2$ & 100\% & $0.99$ \\
                 & $\gamma_\psi$             & 100\% & $+9.2$  & 100\% & $0.43$ \\
                 & static                    & 100\% & $+3.5$  & 25\%  & $0.10$ \\
                 & random                    & 45\%  & $-0.2$  & --    & -- \\
\hline
$-d$(EE, mug$_1$) & simulator gradient       & 100\% & $+9.0$  & 100\% & $1.00$ \\
                 & $D\phi^\top\lambda_\psi$  & 95\%  & $+8.9$  & 95\%  & $1.00$ \\
                 & $\gamma_\psi$             & 100\% & $+5.3$  & 60\%  & $0.35$ \\
                 & static                    & 45\%  & $-0.0$  & 20\%  & $-0.01$ \\
                 & random                    & 52\%  & $+0.2$  & --    & -- \\
\hline
$-d$(EE, mug$_2$) & simulator gradient       & 100\% & $+9.4$  & 90\% & $1.00$ \\
                 & $D\phi^\top\lambda_\psi$  & 95\%  & $+9.3$  & 95\%  & $0.99$ \\
                 & $\gamma_\psi$             & 60\% & $+1.7$  & 30\%  & $0.11$ \\
                 & static                    & 95\%  & $+2.9$  & 45\%  & $0.21$ \\
                 & random                    & 42\%  & $-0.4$  & --    & -- \\
\hline
\end{tabular}
\end{table}

\begin{figure}[t]
\centering
\includegraphics[width=\linewidth]{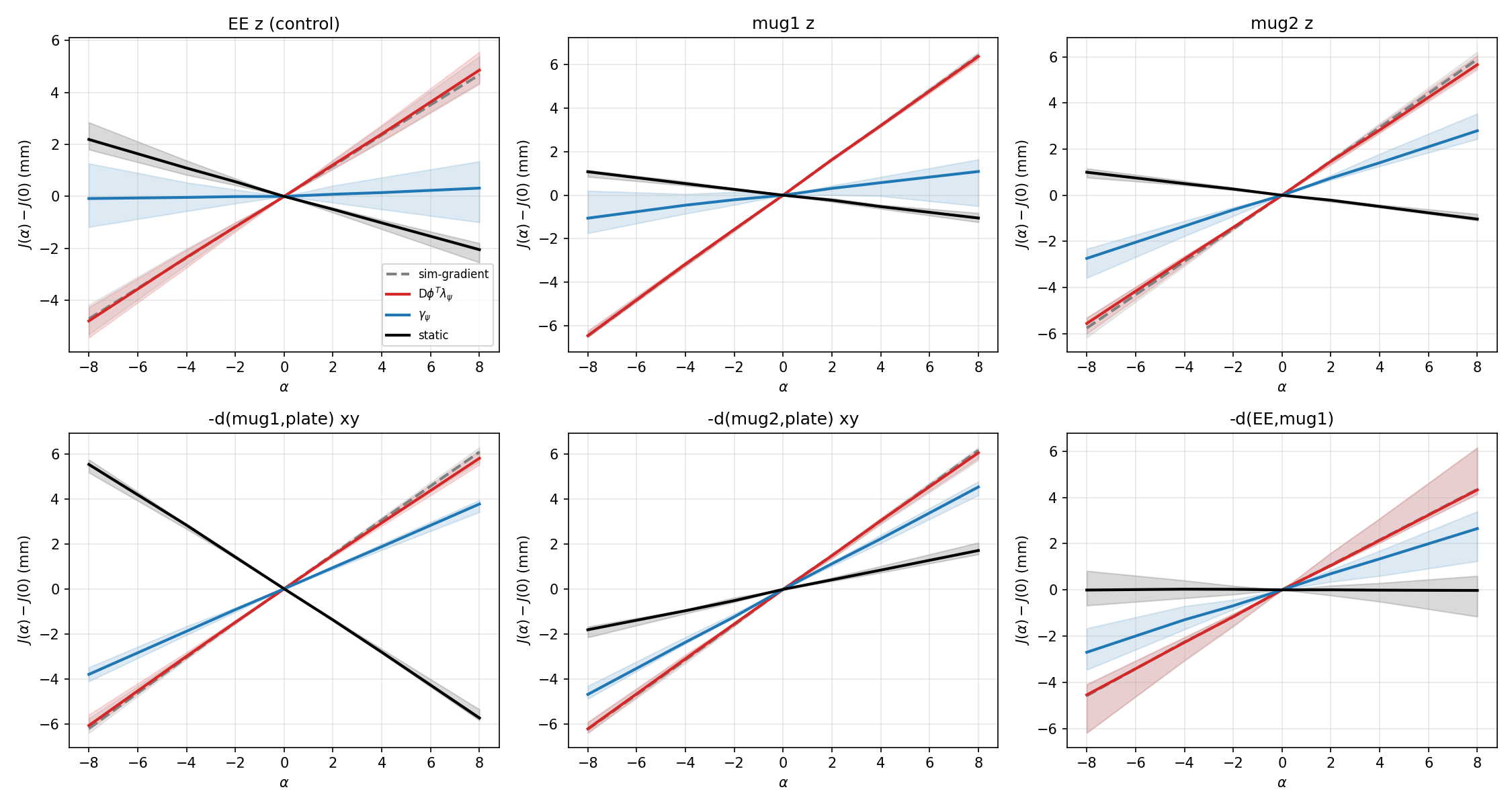}
\caption{\textbf{Steering a single chunk along probe directions.} Change of the QoI after the chunk relative to $\alpha=0$, $J_{\psi, x}(\alpha)-J_{\psi, x}(0)$, as the latent is moved along each direction; median and interquartile range over 20 held-out states per QoI. The probe direction transported through the action head, $D\phi^\top\lambda_\psi$ (red), coincides with the simulator gradient (grey, dashed) for all QoIs. The probe direction without the Jacobian, $\gamma_\psi$ (blue), has the right sign but a smaller and QoI-dependent effect. The static-probe direction (black) decreases the three heights and the negative distance of the white mug to its plate, i.e.\ it steers the wrong way. The task-level QoIs are negative distances, so an increase means moving towards the goal.}
\label{fig:oft_curves}
\end{figure}

\paragraph{A linear steering direction exists.}
The simulator gradient raises the QoI monotonically at all 140 held-out states, by $+9$ to $+13$\,mm over the range of $\alpha$ (\Cref{tab:oft_steering}, \Cref{fig:oft_curves}). Thus, within the leading 32 principal components, the representation of OpenVLA-OFT admits a linear steering direction in the sense of \Cref{def:linear_steering} for each of the seven QoIs, including the four task-level distances, which are nonlinear functions of the object positions.

\paragraph{The probe direction, transported through the head, finds it.}
$D\phi^\top\lambda_\psi$ is monotone at 138 of the 140 states and reaches the effect of the oracle for every QoI (e.g.\ $+12.9$ vs.\ $+12.9$\,mm for the white mug's height and $+12.2$ vs.\ $+12.4$\,mm for the distance of the yellow-and-white mug to its plate), with median cosine $0.98$--$1.00$ to the simulator gradient. The probe also predicts the size of the response: the measured slope over the probe-predicted slope has median between $1.01$ and $1.26$, i.e.\ the probe slightly underestimates the response for some QoIs, consistent with the shrinkage of the ridge coefficients. The two failures (one state for each of the two EE--mug distances) are states at which the probe is inaccurate for the induced chunk variations. In words: the probe correctly predicts how the head-induced change of the chunk moves the QoI, and combining it with the head's Jacobian recovers the simulator's own gradient without querying the simulator.

\paragraph{The static direction is not a steering direction.}
The effect of the static probe direction depends on the QoI but is remarkably consistent across states for a given QoI. It lowers all three heights on median ($-2$ to $-4$\,mm) and increases the distance of the white mug to its plate at all 20 states (between $-6$ and $-12$\,mm on the negative-distance QoI,
median $-11.3$\,mm; median cosine $-0.61$ to the simulator gradient). For the yellow-and-white mug it moves towards the plate, but with a small effect ($+3.5$\,mm, cosine $0.10$), and for the two EE--mug distances it is essentially orthogonal to the simulator gradient (cosine $-0.01$) and rarely beats the
random directions. Changing what the representation encodes about the current value therefore does not make the policy produce the corresponding change, and for several QoIs it produces the opposite. A possible reading is that the policy compensates for a perceived state; we have not tested this interpretation.

\paragraph{The probe gradient alone is not enough.}
For the heights, $\gamma_\psi$ is nearly orthogonal to the simulator gradient for the EE and the white mug (median cosine $0.03$ and $0.13$) and performs close to random; for the yellow-and-white mug it is monotone but reaches about half of the oracle effect. For the mug--plate distances and the distance of the EE to the white mug it does better: it is monotone at all states and reaches $59$--$74\%$ of the
oracle effect, with median cosine $0.35$--$0.43$; for the distance of the EE to the yellow-and-white mug it is weak again ($+1.7$ vs.\ $+9.4$\,mm). It remains clearly below $D\phi^\top\lambda_\psi$ in every case. This is the expected consequence of~\Cref{eq:oft_chain}: without $D\phi$, the direction changes how the
probe reads the representation, but not necessarily how the head changes the action; how much of the right direction it retains varies with the QoI, for reasons we have not established.

\paragraph{Relation to~\Cref{thm:linear_steering}.}
OpenVLA-OFT is deterministic, so the exponential-tilting mechanism of \Cref{thm:linear_steering}, which yields a monotone increase with slope $\mathrm{Var}[\psi]$ along $\lambda_\psi$ itself, does not apply. What the experiment establishes is weaker and complementary: the gradient of the probe-predicted QoI, $D\phi^\top\lambda_\psi$, is an effective linear steering direction in the sense of \Cref{def:linear_steering}, and it coincides with the simulator's gradient. The comparison with $\gamma_\psi$ shows which ingredient of \Cref{thm:linear_steering} is missing in a learned VLA: the identification $D\phi=\nabla^2A$ that lets the probe direction be used as is. With a learned head this identification is replaced by the head's Jacobian, which the model itself provides.

\end{document}